\documentclass{article}

\usepackage[margin=1in]{geometry}

\usepackage[utf8]{inputenc}
\usepackage{authblk}
\usepackage{amsmath}

\title{\bf First Provable Guarantees for Practical Private FL:\\ Beyond Restrictive Assumptions}

\author{Egor Shulgin \quad Grigory Malinovsky \quad Sarit Khirirat \quad Peter Richt{\'a}rik}

\affil{King Abdullah University of Science and Technology (KAUST), Thuwal, Saudi Arabia}

\date{}

\usepackage[round]{natbib}

\usepackage[utf8]{inputenc} 
\usepackage[T1]{fontenc}    
\usepackage{url}            
\usepackage{booktabs}       
\usepackage{amsfonts}       
\usepackage{nicefrac}       
\usepackage{microtype}      
\usepackage[dvipsnames]{xcolor}         

\usepackage{enumitem}  

\usepackage{algorithmic}

\definecolor{darkscarlet}{rgb}{0.34, 0.01, 0.1}
\definecolor{yaleblue}{rgb}{0.06, 0.3, 0.57}
\definecolor{darkpowderblue}{rgb}{0.0, 0.2, 0.6}
\definecolor{midnightblue}{HTML}{0059b3}
\definecolor{noonblue}{HTML}{e5eef7}
\definecolor{chromered}{HTML}{f14233}
\definecolor{olivedrab}{HTML}{6b8e23}

\usepackage[backref=page]{hyperref}
\hypersetup{
	colorlinks=true,
	citecolor=darkscarlet,
	linkcolor=darkpowderblue,
	filecolor=magenta,      
	urlcolor=yaleblue,
	menucolor=gray
}

\renewcommand*{\backrefalt}[4]{%
	\ifcase #1 \footnotesize{(Not cited.)}%
	\or        \footnotesize{(Cited on page~#2)}%
	\else      \footnotesize{(Cited on pages~#2)}%
	\fi}

\usepackage{xspace}

\newcommand{\ouralg}{\textcolor{midnightblue!70!black}{\small\sf Fed-$\alpha$-NormEC}\xspace}

\newcommand{\prioralg}{\textcolor{midnightblue!70!black}{\small\sf $\alpha$-NormEC}\xspace}

\usepackage{subcaption}

\usepackage{multirow}

\usepackage[disable]{todonotes}

\usepackage{amsmath}
\usepackage{amsthm}
\usepackage{amssymb}
\usepackage{amsfonts}
\usepackage{mathtools}

\usepackage{cleveref} 

\usepackage{url}
\usepackage{color}
\usepackage{graphicx}
\usepackage{algorithmic}
\usepackage{algorithm}
\usepackage{verbatim}
\usepackage{appendix}

\newcommand{\Normalize}[2]{{\rm Norm}_{#1}\left( #2 \right)}

\newcommand{\norm}[1]{\left\| #1 \right\|}
\newcommand{\sqnorm}[1]{\left\| #1 \right\|^2}
\newcommand{\inp}[2]{\left\langle#1,#2\right\rangle} 

\newcommand{\cD}{\mathcal{D}}

\newcommand{\cO}{\mathcal{O}}

\newcommand{\del}[1]{}

\definecolor{junglegreen}{rgb}{0.16, 0.67, 0.53}
\definecolor{lasallegreen}{rgb}{0.03, 0.47, 0.19}
\definecolor{midnightblue}{HTML}{0059b3}

 \newcommand{\algname}[1]{\textcolor{midnightblue!70!black}{\small\sf  #1}\xspace}

\newcommand{\R}{\mathbb{R}} 

\newcommand{\eqdef}{:=} 

\newcommand{\Exp}[1]{{\rm E}\left[#1\right]}

\newtheorem{theorem}{Theorem}        
\newtheorem{lemma}{Lemma}        
\newtheorem{corollary}{Corollary}
\newtheorem{definition}{Definition}

\newtheorem{assumption}{Assumption}

\newtheorem*{theorem*}{Theorem}
\newtheorem*{corollary*}{Corollary}

\usepackage[many]{tcolorbox}

\tcolorboxenvironment{theorem*}{
	colback=gray!10,
	boxrule=0pt,
	boxsep=1pt,
	left=2pt,right=2pt,top=2pt,bottom=2pt,
	oversize=2pt,
	sharp corners,
	before skip=10pt,
	after skip=10pt,
}

\tcolorboxenvironment{corollary*}{
	colback=gray!10,
	boxrule=0pt,
	boxsep=1pt,
	left=2pt,right=2pt,top=2pt,bottom=2pt,
	oversize=2pt,
	sharp corners,
	before skip=10pt,
	after skip=10pt,
}

\tcolorboxenvironment{theorem}{
	colback=gray!10,
	boxrule=0pt,
	boxsep=1pt,
	left=2pt,right=2pt,top=2pt,bottom=2pt,
	oversize=2pt,
	sharp corners,
	before skip=10pt,
	after skip=10pt,
}

\tcolorboxenvironment{remark}{
	colback=gray!10,
	boxrule=0pt,
	boxsep=1pt,
	left=2pt,right=2pt,top=2pt,bottom=2pt,
	oversize=2pt,
	sharp corners,
	before skip=10pt,
	after skip=10pt,
}

\tcolorboxenvironment{lemma}{
	colback=gray!10,
	boxrule=0pt,
	boxsep=1pt,
	left=2pt,right=2pt,top=2pt,bottom=2pt,
	oversize=2pt,
	sharp corners,
	before skip=10pt,
	after skip=10pt,
}

\tcolorboxenvironment{corollary}{
	colback=gray!10,
	boxrule=0pt,
	boxsep=1pt,
	left=2pt,right=2pt,top=2pt,bottom=2pt,
	oversize=2pt,
	sharp corners,
	before skip=10pt,
	after skip=10pt,
}

\tcolorboxenvironment{algorithms}{
	colback=gray!10,
	boxrule=0pt,
	boxsep=1pt,
	left=2pt,right=2pt,top=2pt,bottom=2pt,
	oversize=2pt,
	sharp corners,
	before skip=10pt,
	after skip=10pt,
}

\definecolor{red}{rgb}{1.0, 0.01, 0.24}

\usepackage{stfloats}

\definecolor{red}{rgb}{0,0,0}
\definecolor{orange}{rgb}{0,0,0}
\definecolor{Emerald}{rgb}{0,0,0}
\definecolor{blue}{rgb}{0, 0, 0}

\begin{document}

\maketitle

\begin{abstract}
Federated Learning (FL) enables collaborative training on decentralized data. Differential privacy (DP) is crucial for FL, but current private methods often rely on unrealistic assumptions (e.g., bounded gradients or heterogeneity), hindering practical application. Existing works that relax these assumptions typically neglect practical FL features, including multiple local updates and partial client participation. We introduce \ouralg, the first differentially private FL framework providing provable convergence and DP guarantees under standard assumptions while fully supporting these practical features. \ouralg integrates local updates (full and incremental gradient steps), separate server and client stepsizes, and, crucially, partial client participation, which is essential for real-world deployment and vital for privacy amplification. Our theoretical guarantees are corroborated by experiments on private deep learning tasks.
\end{abstract}

\section{Introduction}

Federated Learning (FL)~\citep{mcmahan2017communication,konevcny2016federated} has been widely adopted for training machine learning models  across multiple collaborative devices without centralized data collection. 
Despite its advantages, FL poses three key challenges. 
The first  challenge is the communication bottleneck caused by unreliable and  slow network connections between the server and the clients \citep{caldas2018expanding}. The second challenge is partial client participation, which arises from huge population and intermittent availability of the clients, thus making the training infeasible to involve all the clients in every communication round \citep{chen2020optimal}.  
The third challenge is data heterogeneity in FL systems, where clients' datasets are diverse and not identically distributed \citep{ karimireddy2020scaffold, mishchenko2022proxskip}. 
This has necessitated the development of FL methods that improve communication efficiency, accommodate partial client participation, and address data heterogeneity \citep{wang2021field, kairouz2021advances}.

FL methods keep  data decentralized across  devices, and thus were initially perceived as privacy-preserving methods. 
However, FL remains vulnerable to various privacy threats. 
Prior studies ~\citep{boenisch2023curious,8835269} have shown that sensitive information can still be inferred from the shared model parameters--either by an untrusted central server or by adversaries performing inference attacks.
To address these privacy concerns, FL methods have been extended to incorporate formal privacy guarantees through  Differential Privacy (DP)~\citep{dwork2014algorithmic}. 
This integration in FL methods is achieved by applying DP mechanisms to 
the information, such as gradients or local updates, before it is shared with the server. 
Typically, this involves applying a clipping or normalization operator to bound the sensitivity of the updates, followed by the addition of carefully tuned DP noise.

Although clipping enables DP in FL methods, it introduces bias that can hinder convergence. 
For instance, even in the absence of DP noise, \algname{FedAvg} with model clipping  fails to converge to the optimum for solving a simple quadratic problem~\citep{zhang2022understanding}.
As a result, convergence guarantees for federated DP methods with clipping are typically established under strong and often unrealistic  assumptions, such as uniformly bounded gradient norms~\citep{zhang2020private,li2022soteriafl,lowy2023private} and/or bounded heterogeneity~\citep{noble2022differentially,li2024improved}.
These assumptions do not often hold in practical FL scenarios,  where client data can be arbitrarily heterogeneous, and they tend to obscure the impact of clipping-induced bias. 
To our knowledge, convergence guarantees for FL DP methods remain elusive unless this  bias is explicitly accounted for.

To eliminate the bias caused by clipping and ensure convergence, Error Compensation (EC), also known as Error Feedback (EF, EF21)~\citep{gorbunov2020linearly,stich2019error, richtarik2021ef21}, has been integrated into distributed methods that use gradient clipping, such as Clip21~\citep{khirirat2023clip21}. 
This technique tracks the error between the clipped gradient and the true gradient, and uses it to correct future clipped gradients, thus effectively reducing the clipping bias as the methods progress. 
In non-private settings, EC ensures the convergence of distributed clipping methods under standard assumptions. 
However, its benefit does not  extend to private settings, where DP noise is added.

To ensure theoretical guarantees in the private setting, two recent distributed differentially private (DP) methods, which operate on gradients, achieve strong convergence and privacy guarantees under standard assumptions without requiring bounded gradients or bounded  heterogeneity.
The first method by \citet{islamov2025double} combines the momentum update with the EC mechanism, while the second method  by \citet{shulgin2025smoothed} replaces clipping in the EC mechanism with smoothed normalization~\citep{yun2021can, bu2023automatic}, the operator that is more robust across parameter choices than standard clipping.  
Although these methods converge in both non-private and private settings, they do not support key components that are essential in practical FL systems, such as multiple local training steps and partial client participation. 
Therefore, formal convergence and privacy guarantees of private FL methods under standard assumptions remain largely unexplored.

\subsection{Contributions}

We summarize our key contributions as follows:

\textbf{$\bullet$ A practical, private FL method.} 
We propose \ouralg, an FL method that extends \prioralg~\citep{shulgin2025smoothed} by integrating smoothed normalization and the error feedback mechanism \algname{EF21} into clients' local update steps, rather than directly into their gradients as in \prioralg.
Unlike \prioralg, \ouralg supports partial client participation and allows local training through multiple gradient steps, while introducing separate step sizes for client-side local updates and server-side global aggregation--providing flexibility in controlling their respective effects. 
Additionally, to reduce the computational cost of full gradient steps at clients, \ouralg incorporates cyclic incremental gradient descent locally, a feature absent in \prioralg.

\textbf{$\bullet$ Convergence guarantees for non-convex, smooth problems under standard assumptions.}
We establish the convergence of \ouralg for minimizing non-convex, smooth objectives without relying on additional restrictive assumptions, such as bounded gradients or bounded heterogeneity. Our analysis encompasses both local gradient descent and incremental gradient updates. 
{\color{red}
A fundamentally new convergence analysis is required because \ouralg applies the EF21 mechanism to the clients' local updates, a shift from~\prioralg that applies it to the gradients. 
Specifically, our novel proof technique must bound the Euclidean distance involving the fixed-point operator $\mathcal{T}_i(x)$  for local updating at client $i$. Ensuring this bound allows us to remove the dependence on additional restrictive assumptions often required by prior works studying private FL methods.
}

\textbf{$\bullet$ Differential privacy guarantees with amplification via partial participation.}
We provide a privacy analysis of \ouralg for both single and multiple local update steps. 
In both cases, we introduce an independent client sampling scheme, where each client participates in each round with sampling probability $p$, independently of the others. Our analysis shows that this partial participation in \ouralg, which differs from \prioralg that relies on full participation, enables a significant reduction in the variance of the DP noise through the mechanism of privacy amplification via subsampling.

\textbf{$\bullet$ Empirical validations of \ouralg~on image classification.}
We demonstrate the effectiveness of \ouralg by applying it to the image classification task on the CIFAR-10 dataset using the ResNet20 model. 
Experiments highlight the impact of key algorithm parameters and client participation levels, corroborating our theoretical insights on convergence and privacy trade-offs. Notably, we show that partial participation, by leveraging privacy amplification, can achieve target accuracy with significantly improved communication efficiency compared to full participation, showcasing \ouralg's utility for real-world private deep learning.

\section{Related Works}

In this section, we review existing literature that are closely related to our work. 

\textbf{Communication efficiency in FL.}
Two key strategies for improving communication efficiency in FL algorithms are local updates and compression. 
First, local updates reduce communication frequency by allowing clients to perform multiple training steps locally before sending the updates to the server. 
This idea has been widely explored in the FL literature, e.g. by~\citet{khaled2020tighter, malinovskiy2020local, koloskova2020unified, gorbunov2021local, patel2024limits}.
Second, compression improves communication efficiency by reducing the size of transmitted messages.
Compression operators can be either unbiased or biased. 
The convergence of FL algorithms with unbiased compression--such as \algname{FedAvg}~\citep{haddadpour2021federated}, local gradient descent~\citep{khaled2019gradient, shulgin2022shifted}, and fixed-point methods~\citep{chraibi2019distributed}--has been extensively studied.
While many FL analyses assume unbiased compression, biased compression has also been investigated. 
For instance, ~\citet{gruntkowska2023ef21} study the use of biased compression in combination with error feedback mechanisms. 
However, 
their approaches focus only on compressing the global iterates maintained at the server side, and are limited to the distributed setting, where clients synchronize with the server after every update.

\textbf{Clipping.}
Two popular clipping operators for FL algorithms are per-sample clipping and per-update clipping. 
Per-sample clipping~\citep{liu2022communication} bounds the norm of the local gradient being used to update the local model parameters on each client, and ensures example-level privacy~\citep{abadi2016deep}. Per-update clipping~\citep{geyer2017differentially} limits the bound of the local model update, and preserves user-level privacy~\citep{zhang2022understanding,geyer2017differentially}, which provides stronger privacy guarantee than example-level privacy. 
The convergence of FL algorithms, such as   
\algname{FedAvg}~\citep{mcmahan2018learning} and \algname{SCAFFOLD}~\citep{karimireddy2020scaffold},  with  per-sample and/or per-update clipping  was analyzed by~\citep{zhang2022understanding,noble2022differentially,li2024improved,liu2022communication,wang2023efficient, shulgin2024convergence}. 
In this paper, we leverage per-update smoothed normalization, introduced by~\citet{bu2023automatic} as an alternative to clipping, to design FL algorithms that accommodate local training and differential privacy.

\textbf{Federated learning with clipping and privacy.}
A simple yet popular FL algorithm, \algname{FedAvg}~\citep{mcmahan2017communication}, has been adapted to provide differential privacy (DP) by clipping model updates and injecting random noise~\citep{mcmahan2018learning,geyer2017differentially,triastcyn2019federated}.
These \algname{DP-FedAvg} algorithms were outperformed by \algname{DP-SCAFFOLD}~\citep{noble2022differentially}, a DP version of \algname{SCAFFOLD}~\citep{karimireddy2020scaffold}.
However, these existing results require restrictive assumptions that do not hold in practice, especially in deep neural network training, such as uniformly bounded stochastic noise~\citep{liu2022communication,crawshaw2023episode}, bounded gradients~\citep{zhang2022understanding,li2022soteriafl,lowy2023private,zhang2020private} (which effectively ignores the impact of clipping bias), and/or bounded heterogeneity~\citep{noble2022differentially,li2024improved}. 
To the best of our knowledge, there has been a recent work by~\citet{das2021convergence} that provides convergence guarantees for \algname{DP-FedAvg} without these restrictive assumptions, but their results are limited to convex, smooth problems and require a stepsize to depend on an inaccessible constant $\Delta_i\eqdef f_i(x^\star) - \min_{x\in\R^d} f_i(x)$, where $x^\star = \arg\min_{x\in\R^d} f(x)$. 
In this paper, we provide convergence guarantees for private FL algorithms with smoothed normalization and error feedback. In particular, our guarantees do not rely on the restrictive assumptions commonly used by prior work, and our theoretical stepsizes can be implemented in practice.

\textbf{Server and local stepsizes.} 
The use of separate server and local stepsizes has been shown to be crucial in federated learning~\citep{charles2020outsized, reddi2020adaptive, malinovsky2023federated}. This separation provides greater flexibility in optimization. The local stepsize helps mitigate the impact of data heterogeneity and controls the variance from local updates~\citep{malinovsky2023server}, while the global (server-side) stepsize manages the aggregation process and stabilizes extrapolation during model updates~\citep{li2024power}.

\textbf{Random reshuffling.}
Random reshuffling, or without-replacement sampling, is used in SGD and often outperforms sampling with replacement. Its convergence properties have been extensively studied~\citep{mishchenko2020random, haochen2019random, safran2021random, yun2021can}, including in FL settings~\citep{mishchenko2022proximal, sadiev2022federated, malinovsky2022federated}. Other without-replacement strategies include Shuffle-Once~\citep{safran2020good} and Incremental Gradient methods~\citep{bertsekas2011incremental, koloskova2023convergence}.
In this work, \ouralg~can be extended to support Incremental Gradient updates, partial participation, and differential privacy with provable convergence.

\textbf{Error feedback.}
Error feedback, also known as error compensation, has proven effective in enhancing the convergence of distributed gradient algorithms with compressed communication. Popular error feedback mechanisms include 
\algname{EF14}~\citep{seide20141}, \algname{EF21}~\citep{richtarik2021ef21}, \algname{EF21-SGDM}~\citep{fatkhullin2023momentum}, \algname{EControl}~\citep{gao2023econtrol}, and \algname{EFSkip}~\citep{bao2025efskip}.
Beyond compression, error feedback has been adapted by substituting compression with other operators. 
For instance, \algname{EF21} has inspired the development of \algname{Clip21}~\citep{khirirat2023clip21} (using clipping instead of compression) and \prioralg~\citep{shulgin2025smoothed} (employing smoothed normalization).
In this paper, we contribute by adapting \prioralg~to the FL setting, resulting in \ouralg.

\section{Preliminaries} \label{sec:preliminaries}

In this section, we provide notations and  problem formulation that will be used throughout this paper.

\textbf{Notations.}
We use $[a,b]$ for the set $\{a,a+1,\ldots,b\}$ for integers $a,b$ such that $a \leq b$,  $\Exp{u}$ for the expectation of a random variable $u$, and $f(x)=\cO(g(x))$ if $f(x) \leq A g(x)$ for some $A>0$ for functions $f,g:\R^d\rightarrow \R$. Finally, for vectors $x,y\in\R^d$, $\inp{x}{y}$ denotes their inner product, and $\norm{x}$ denotes the Euclidean norm of $x$.

\textbf{Federated optimization problem.}
Consider an FL setting with the server being connected with $M$ clients over the network. 
Each client $i \in [1,M]$ has a private dataset.
The objective is to determine the vector of model parameters $x\in\R^d$ that solves the following optimization problem:
\begin{eqnarray}\label{eqn:problem}
    \underset{x\in\R^d}{\text{minimize}} \quad f(x) = \frac{1}{M}\sum_{i=1}^M f_i(x),
\end{eqnarray}
where $f_i(x) \eqdef \frac{1}{N}\sum_{j=1}^N f_{i,j}(x)$, and  $f_{i,j}(x)$ is the loss of the model parameterized by $x$ on training data $j \in [1,N]$ of client $i \in [1,M]$. Also, we assume that the objective functions $f$, $f_i$, and $f_{i,j}$ satisfy the following conditions, which are standard for analyzing federated algorithms.
\begin{assumption}\label{assum:smooth}
Consider Problem~\eqref{eqn:problem}. Let each individual function $f_{i,j}(x)$ be $L$-smooth and bounded below by $f_{i,j}^{\inf} > -\infty$; let each local function $f_i(x)$ be bounded below by $f_i^{\inf} > -\infty$; and let the global objective $f(x)$ be bounded below by $f^{\inf} > -\infty$.
\end{assumption}

\textbf{DP-FedAvg.} 
The simplest FL algorithm for solving Problem~\eqref{eqn:problem} is \algname{DP-FedAvg}~\citep{mcmahan2018learning}.
The algorithm contains two steps:  model updating on each client and model aggregation on the server. The server updates the next global model vector $x^{k+1}$ via: 
\begin{eqnarray}
    x^{k+1} = x^k - \frac{\eta}{B} \left[\sum_{i\in S^k } \Psi(x^k - \mathcal{T}_i(x^k)) + z_i^k \right],
\end{eqnarray}
where $S^k$ is the subset of $[1,M]$ with size $B \leq M$,  $\Psi(\cdot)$ is a bounding operation such as clipping or normalization, $\mathcal{T}_i(\cdot)$ is the fixed point operator representing the local update performed by client $i$ based on the current global model $x^k$ and its private data associated with the local function $f_i(\cdot)$, and 
$z_i^k\in\R^d$ is the DP noise.
Since $\Psi(\cdot)$ constrains the magnitude of the model update $\mathcal{T}_i(x^k) - x^k$, we can calibrate the variance of the DP noise $z_i^k$ proportionally to this bound to achieve the desired privacy guarantees.
Moreover, the fact that only a subset of $B$ clients communicate with the server in each round leads to a significant reduction in the required noise variance due to the privacy amplification effect of subsampling.

\textbf{Bias from Clipping or Normalization.}
Clipping and normalization inherently introduce bias, causing \algname{DP-FedAvg} to generally not converge, even without the addition of the DP noise. For instance, \citet[Claim 2.1]{zhang2022understanding} demonstrates that \algname{FedAvg} with model clipping fails to converge to the global optimum when solving a convex quadratic problem.
Existing analyses of \algname{DP-FedAvg} often circumvent the impact of this clipping bias by assuming bounded gradients~\citep{zhang2022understanding,li2022soteriafl,lowy2023private,zhang2020private}.
Acknowledging this limitation,~\citet{das2021convergence} recently attempts to analyze the convergence of \algname{DP-FedAvg} without relying on the bounded gradient assumption. However, their findings are restricted to convex and smooth problems, and require a step size to know the constant that cannot be accessed in practice: $\Delta_i\eqdef f_i(x^\star) - \min_{x\in\R^d} f_i(x)$, where $x^\star = \arg\min_{x\in\R^d} f(x)$.

{\color{red}
\textbf{Client-level differential privacy.} 
In FL, we wish to prevent information leakage about each client's entire dataset from the exchanged updates. 
We adopt client-level differential privacy~\citep{li2022soteriafl}: 
\begin{definition}[Client-level DP]
A randomized mechanism $\mathcal{M}:\cD\rightarrow\mathcal{R}$ with domain $\mathcal{D}$ and range $\mathcal{R}$ is $(\epsilon,\delta)$-DP at the client level if for every pair of client-neighboring datasets $D,D^{\prime} \in \cD$ and for all events $S\in\mathcal{R}$ in the output space of $\mathcal{M}$,  we have 
\begin{eqnarray*}
    {\rm Pr}(\mathcal{M}(D) \in S) \leq e^{\epsilon} {\rm Pr}(\mathcal{M}(D^{\prime}) \in S) + \delta.
\end{eqnarray*}
\end{definition}
}

\section{Fed-$\alpha$-NormEC}\label{sec:fed_normec}

To resolve theoretical limitations in existing private FL algorithms, we propose~\ouralg (\Cref{alg:norm_norm21_dp_v1-fed}). 
Our method achieves this by leveraging smoothed normalization, the EF21 mechanism, local updating, and partial client participation. 
At each communication round $k = 0, 1, \ldots, K$, 
\ouralg perform the following key updating rules: 

\begin{enumerate}[label=\textbf{\arabic*)}, leftmargin=*, align=left]    
\item The server broadcasts the current global model $x^k$ to a subset of participating clients. 
Specifically, client participation is determined independently for each client by a fixed probability $p$ in each round, which models the partial client participation. 
    \item  Every client performs a local update based on the received model using a  fixed-point operator $\mathcal{T}_i(x^k)$, which involves multiple gradient descent steps.
Also, each client computes its local memory vector $v_i^k$ to capture its local model update $\nicefrac{(x^k-\mathcal{T}_i(x^k))}{\gamma}$  using  smoothed normalization\footnote{Smoothed normalization guarantees bounded sensitivity of the client update. That is, it ensures  
$\|\operatorname{Norm}_\alpha(v)\| \leq 1$ for any $v \in \R^d$,~\citet{shulgin2025smoothed}.}  
$$\operatorname{Norm}_\alpha(v) := \frac{1}{\alpha+\|v\|} v,$$
with $\alpha \geq 0$ and 
the EF21 mechanism via: 
\begin{eqnarray*}
	v_i^{k+1} = v_i^k +\beta \Normalize{\alpha}{\frac{x^k-\mathcal{T}_i(x^{k})}{\gamma} - v_i^{k}},
\end{eqnarray*}	 
where $\beta>0$ controls the update of error feedback and $\gamma>0$ is a local stepsize associated with local update operator $\mathcal{T}_i(x)$.
\item  Each participating client transmits its update $\hat \Delta_i^k$ to the server. 
In the non-private setting, $\hat \Delta_i^k \coloneqq q_i^k \Normalize{\alpha}{\frac{x^k - \mathcal{T}_i(x^{k})}{\gamma} - v_i^k}$, while in the private setting, $\hat \Delta_i^k \coloneqq q_i^k \left( \Normalize{\alpha}{\frac{x^k - \mathcal{T}_i(x^{k})}{\gamma} - v_i^k} + z_i^k \right)$.
Here, $q_i^k$ is equal to $1/p$ with probability $p$, and 0 otherwise. The DP noise vector $z_i^k$ is sampled from a Gaussian distribution with zero mean and variance $\sigma_{\text{DP}}^2$.
\item  The server aggregates the normalized local update vectors received from the participated clients and computes the global memory vector $\hat v^k$ and the server updates the global model $x^{k+1}$ using the normalized step as follows:
\begin{eqnarray*}
	\hat v^{k+1} = \hat v^k + \frac{\beta}{M}\sum_{i=1}^M \hat \Delta_i^k, \quad 	x^{k+1}  =x^k - \frac{\eta}{\norm{ \hat v^{k+1}}}\hat v^{k+1},
\end{eqnarray*}	
where $\eta>0$ is the server-side stepsize.
\end{enumerate}

\begin{algorithm}[t]
	\caption{\algname{(DP-)}\ouralg}
	\label{alg:norm_norm21_dp_v1-fed}
	\begin{algorithmic}[1]
		\STATE \textbf{Input:} Tuning parameters $\gamma>0$,  $\beta>0$, and $\eta \in (0, 1)$; normalization parameter $\alpha>0$;  initialized vectors $x^0,{v_i^{0}} \in \R^d$ for $i\in [1,M]$ and $\hat v^{0} = \frac{1}{M}\sum_{i=1}^M v_i^{0}$; local fixed-point operators $\mathcal{T}_i(\cdot)$; probability of transmitting the client's local vector to the server $p\in [0,1]$; Gaussian noise with zero mean and $\sigma_{\rm DP}^2$-variance $z_i^k \in \R^d$.
		\FOR{each iteration $k = 0, 1, \dots, K$}
        \STATE Select $S^k$ with sampling probability $p$
		\FOR{clients $i \in [1,M]$ in parallel}
		\STATE Compute local update  $\mathcal{T}_i(x^k)$
		\STATE Compute $\Delta_i^k = \Normalize{\alpha}{\frac{x^k - \mathcal{T}_i(x^{k})}{\gamma} - v_i^{k}}$
		\STATE Update  {$v_i^{k+1} = v_i^{k} + \beta \Delta_i^k$}
        \ENDFOR
  		\FOR{clients $i \in S^k$ in parallel}
		\STATE Set $q_i^k = 1/p$ 
		\STATE \textbf{Non-private setting:}  Transmit $\hat \Delta_i^k = q_i^k \Delta_i^k$
		\STATE \textbf{Private setting:}  Transmit $\hat \Delta_i^k = q_i^k (\Delta_i^k + z_i^k)$
		\ENDFOR
		\STATE Server computes $\hat v^{k+1} = \hat v^{k} + \frac{\beta}{M}\sum_{i=1}^M \hat \Delta_i^k$
		\STATE Server updates {$x^{k+1} = x^k - \frac{\eta}{\norm{ \hat v^{k+1}}} \hat v^{k+1} $} 
		\ENDFOR
		\STATE \textbf{Output:} $x^{K+1}$
	\end{algorithmic}
\end{algorithm}

{\color{red}
\textbf{Comparison to \prioralg and DP-SGD.} 
\ouralg adapts \prioralg to FL settings, where local updates and partial client participation are key components. 
Unlike \prioralg, which applies the normalization operator and EF21 mechanism to the gradients, \ouralg applies them to the  model updates $\nicefrac{( x^k-\mathcal{T}_i(x^k))}{\gamma}$.
Furthermore, \ouralg adopts the partial client participation protocol from DP-SGD \citep{abadi2016deep}, where gradients are exchanged over the server-client network.
}

\section{Convergence Results}\label{sec:thm}

To this end, we provide the convergence theorem for \ouralg that incorporates multiple local gradient descent (GD) steps and partial participation in the private setting. 
Here, the local fixed-point operator $\mathcal{T}_i(\cdot)$ is defined as: $
\mathcal{T}_i(x^k) = x^k - \gamma \cdot \frac{1}{T} \sum_{j=0}^{T-1} \nabla f_{i}(x_i^{k,j})$, where  the sequence $\{x_i^{k,j}\}$ is generated by
$$
x^{k,j+1}_i = x^{k,j}_i - \frac{\gamma}{T} \nabla f_{i}(x^{k,j}_i) \quad \text{for} \quad j=0,1,\ldots,T-1,
$$
given $x_i^{k,0}=x^k$.

\begin{theorem}[\ouralg with local GD steps]\label{thm:DP_PP_full}
Consider \ouralg~for solving Problem~\eqref{eqn:problem} where~\Cref{assum:smooth} holds.    
Let $\beta,\alpha>0$ be chosen such that $\frac{\beta}{\alpha+R} < 1$ with $R = \max_{i \in [1,M]} \norm{v_i^0 - \frac{x^0 - \mathcal{T}_i(x^0)}{\gamma}}$.
{\color{blue} If $\gamma = \frac{1}{2L}$ and $\eta \leq \min\left( \frac{1}{K+1}  \frac{\Delta^{\inf}}{2\sqrt{2L}} , \frac{1}{4L}\frac{\beta R}{\alpha + R} \right)$, then}
\begin{align*}
    \underset{k \in [0,K]}{\min} \Exp{\norm{\nabla f(x^k)}}  \leq \frac{3}{K+1} \frac{f(x^0)-f^{\inf}}{\eta} + 2R  + 2\sqrt{\frac{\beta^2 B}{M}(K+1)}  + \frac{\eta L}{2} + \gamma \cdot   \mathbb{I}_{T\neq 1} \left[8L\sqrt{2L}  \sqrt{\Delta^{\inf}} \right],
\end{align*}
where 
$B= 2\frac{(p-1)^2}{p} + 2(1-p) + 2\sigma^2_{\rm DP}/p$, and $\Delta^{\inf} = f^{\inf} - \frac{1}{M}\sum_{i=1}^M f_i^{\inf} \geq0$. 
\end{theorem}

\textbf{Discussion on~\Cref{thm:DP_PP_full}.}
From~\Cref{thm:DP_PP_full},
\ouralg with multiple local GD steps achieves sub-linear convergence with additive constants arising from smoothed normalization $R$, partial participation and private noise $
B = 2\frac{(p - 1)^2}{p}+2(1-p) + \frac{2\sigma^2_{\rm DP}}{p}$,  data heterogeneity $\Delta^{\inf}$.
In contrast to \prioralg~\citep{shulgin2025smoothed}, \ouralg naturally handles the local updates and partial participation protocol.
{\color{red}
Moreover, the convergence bound of \ouralg contains the heterogeneity metric $\Delta^{\inf} = f^{\inf} - \frac{1}{M}\sum_{i=1}^M f_i^{\inf}$.
{\color{blue}
This metric can be small for many learning problems. 
In particular, \(\Delta^{\inf} = 0\) for problems with non-negative losses, such as squared error loss in linear regression, logistic loss in logistic regression, and cross-entropy loss in neural network training.
}
Our results are in stark contrast to many analyses for FL algorithms, e.g. by~\citet{noble2022differentially, li2024improved}, that  assume uniformly bounded heterogeneity, i.e. $\|\nabla f_i(x) - \nabla f(x)\|^2 \le \xi^2$ for some $\xi >0$.
This bounded heterogeneity assumption is often unrealistic for many settings. For instance, it does not hold for  minimizing quadratic objective functions with heterogeneous data over an unbounded domain. 
}
{
\color{red} \textbf{Novel analysis techniques for \ouralg.}  
Our convergence analysis for \ouralg departs from that of \prioralg~\citep{shulgin2025smoothed}, as it relies on two key error bounds due to client drifts by local updates and data heterogeneity. 
Specifically, we bound (1) $\norm{x^k-\gamma \nabla f_i(x^k) - \mathcal{T}_i(x^k)}$ using~\Cref{lemma:localGD_avg_gradient_norm}, and (2)  $\norm{\nabla f_i(x^k)}$ using ~\Cref{lemma:localGD_recursion}, which depends on the heterogeneity metric $\Delta^{\inf}$.
This metric allows us to handle arbitrary levels of data heterogeneity. 
}

In the subsequent sections, based on~\Cref{thm:DP_PP_full}, we present the convergence of \ouralg that use either a single local update step or multiple steps.

\subsection{One Local GD Step}

We begin by analyzing \ouralg with one local GD step, i.e. with $\mathcal{T}_i(x) = x - \gamma \nabla f_i(x),  \text{for } i \in [1, M]$.

\textbf{Full participation and non-private setting.}
In a full participation and non-private setting, 
according to~\Cref{thm:DP_PP_full} with $T=1$ (one local GD step), $p=1$ (full client participation), and $\sigma_{\rm DP}=0$ (no DP noise), \ouralg with one local GD step achieves the convergence bound consisting of three terms:
$\frac{3}{K+1}  \frac{f(x^0) - f^{\inf}}{\eta} + 2R + \frac{\eta L}{2}.$ 
By choosing $R = \cO(1/\sqrt{K+1})$ and $\eta = \cO(1/\sqrt{K+1})$,  \ouralg with one local GD step achieves the $\cO(1/\sqrt{K+1})$ convergence rate in the gradient norm, which matches that of \prioralg by \citet[Corollary 1]{shulgin2025smoothed} and that of classical gradient descent.
The upper-bound for $R$ can be achieved, e.g., by setting $v_i^0=\nicefrac{(x^0-\mathcal{T}_i(x^0))}{\gamma} + e$, where $e = (D/\sqrt{K+1},0,0,\ldots,0)\in\R^d$.

\textbf{Partial participation and private setting.}
In the partial participation and private setting, \ouralg with one local GD step can be shown to be $(\epsilon,\delta)$-DP at the client level, by leveraging the instance-level DP guarantee of \citet{abadi2016deep}, which uses the moments accountant analysis on DP-SGD. 
This is justified because \ouralg uses the same partial participation protocol as DP-SGD, where at each communication round a random subset of clients is sampled to participate. 
Instead of per-example clipping in DP-SGD, \ouralg applies normalization to each client's local update to enforce the sensitivity of $1$. 
\ouralg injects zero-mean Gaussian noise to the aggregated update, which is chosen with the same scaling rule according to~\citet{abadi2016deep}.
By treating each client’s normalized update as analogous to a clipped per-example gradient in DP‑SGD, the instance-level privacy guarantee by~\citet{abadi2016deep} can be directly translated into the client-level setting.
Specifically, we set $\sigma_{\rm DP} = c \cdot \frac{p \sqrt{(K+1)\log(1/\delta)}}{\epsilon}$ for some constant $c > 0$ and $0 < p \leq 1$. Notably, $\sigma_{\rm DP}$ exhibits a reduced dependency on $p$ thanks to the amplification effect of subsampling in the client-level privacy setting. The formal utility guarantee is given below:

\begin{corollary}
\label{corr:DP_PP_utility}
 Consider  \ouralg~ for solving Problem~\eqref{eqn:problem} under the same setting as~\Cref{thm:DP_PP_full}. 
 Let $T=1$ (one local GD step),  let $\sigma_{\rm DP}= c \frac{p \sqrt{(K+1)\log(1/\delta)}}{\epsilon}$ with $c >0$ (privacy  with subsampling amplification), and let $p=\frac{\hat B}{M}$ for $\hat B\in[1,M]$ (client subsampling).
If $\beta = \frac{\hat \beta}{K+1}$ with $\hat \beta = \sqrt{\frac{3(f(x^0)-f^{\inf})}{\gamma}}\sqrt[4]{\frac{M}{B_2}}$, 
$\gamma < \frac{\Delta^{\inf} (\alpha + R)}{\sqrt{2L} \hat\beta R}$
$\alpha = R = \cO\left( \sqrt[4]{d}\frac{\sqrt{f(x^0)-f^{\inf}}}{\sqrt{\gamma}}\sqrt[4]{\frac{B_2}{M}} \right)$ with $B_2 = 2c^2 \frac{\hat B}{M} \frac{ \log(1/\delta)}{\epsilon^2}$, and 
$\eta = \frac{1}{K+1}    \frac{\gamma}{2}\frac{\hat \beta R}{\alpha + R}$, then
\begin{eqnarray*}
     \underset{k \in [0,K]}{\min} \Exp{\norm{\nabla f(x^k)}}  = \cO\left(\Delta \sqrt[4]{\frac{d\hat B}{M^2}\frac{\log(1/\delta)}{ \epsilon^2}} \right), 
\end{eqnarray*}
where $\Delta = \max(\alpha,2)\sqrt{L} \sqrt{f(x^0)-f^{\inf}}$.
\end{corollary}

\textbf{Discussion on~\Cref{corr:DP_PP_utility}.}
\Cref{corr:DP_PP_utility} establishes the utility bound of \ouralg in the partial participation and private  setting.
By setting $p = \hat{B}/M$, where $\hat{B} \in [1, M]$ denotes the number of clients sampled at each round, \ouralg achieves a utility bound of $
\mathcal{O}\left( \Delta \sqrt[4]{d \cdot \frac{B}{M^2} \cdot \frac{\log(1/\delta)}{\epsilon^2}} \right),
$ which improves upon the $
\mathcal{O}\left( \Delta \sqrt[4]{d \cdot \frac{1}{M} \cdot \frac{\log(1/\delta)}{\epsilon^2}} \right)
$ utility bound of \prioralg. This improvement arises due to privacy amplification via subsampling induced by partial participation.
Finally, when $p = 1$ (full client participation), \ouralg recovers the same utility bound as \prioralg.
Moreover, we can ensure $\hat v^0 = \frac{1}{M}\sum_{i=1}^M v_i^0$ by using secure aggregation techniques.
Specifically, each client adds cryptographic noise before their updates are communicated to the server, thus maintaining this initialization condition while preserving privacy. We refer to~\citet{shulgin2025smoothed} for a comprehensive discussion on this initialization approach.

\subsection{Multiple Local GD Steps}
Next, we analyze the convergence of \ouralg with multiple local GD  steps in a partial participation and private setting. 
According to~\Cref{thm:DP_PP_full} with $T>1$, the convergence bound of \ouralg includes one additional error term:  $\mathbb{I}_{T \neq 1} \left[ 8L\sqrt{2L} \cdot \sqrt{\Delta^{\inf}} \right]$,
which stems from client drift due to local updates $T$ and data heterogeneity measured by $\Delta^{\inf} = f^{\inf} - \frac{1}{M} \sum_{i=1}^M f_i^{\inf}$.
This term indicates that \ouralg becomes more effective as client data becomes less heterogeneous.  
Moreover, it can be controlled, or even eliminated, by reducing the local stepsize $\gamma$.  
It is important to note that $\Delta^{\inf}$ can be arbitrarily large, in contrast to the uniform heterogeneity bounds often imposed by~\citet{noble2022differentially,li2024improved}.  
However, $\Delta^{\inf}$ may vanish for homogeneous datasets.
In extreme cases, $\Delta^{\inf}=0$ when all clients share the same infimum, i.e., $f^{\inf}_1 = f^{\inf}_2 = \ldots = f^{\inf}_M$.
Furthermore, by applying~\Cref{thm:DP_PP_full} with $T>1$, and by  using the variance of the DP noise $\sigma_{\rm DP} = c \cdot \frac{p \sqrt{(K+1)\log(1/\delta)}}{\epsilon}$ for some constant $c > 0$ and $0 < p \leq 1$, \ouralg achieves the utility guarantee under properly tuned parameters, as shown in~\Cref{corr:DP_PP_utility_local_GD}.

\subsection{Extension to Local IG Steps} 
To reduce the cost of computing full gradients on clients, we introduce a variant of \ouralg that uses cyclic incremental gradient (IG) steps\footnote{\color{orange}We focus on the cyclic IG steps to avoid high-probability analyses typically required for methods using clipping or normalization. Extension to random reshuffling and arbitrary numbers of epochs is left for future work. 
}. 
In this variant, each client performs local updates using gradient steps on its individual component functions $f_{i,j}$, applied in a cyclic order over its local dataset. 
Here, the local fixed-point operator $\mathcal{T}_i(\cdot)$ is defined as: $
\mathcal{T}_i(x^k) = x^k - \gamma \cdot \frac{1}{N} \sum_{j=0}^{N-1} \nabla f_{i,j}(x_i^{k,j})$, where  the sequence $\{x_i^{k,j}\}$ is generated by
$$
x^{k,j+1}_i = x^{k,j}_i - \frac{\gamma}{N} \nabla f_{i,j}(x^{k,j}_i) \quad \text{for} \quad j=0,1,\ldots,T-1,
$$
with initialization $x_i^{k,0}=x^k$. 
The convergence bound of \ouralg using local cyclic IG steps in \Cref{thm:FedNormEC_IG_PP_DP_multiple} introduces an additional error term of
$\gamma \cdot 4L\sqrt{2L} \cdot \sqrt{ \frac{1}{M} \sum_{i=1}^M \Delta^{\inf}_i },$
where $\Delta^{\inf}_i = f^{\inf} - \frac{1}{N} \sum_{j=1}^N f_{i,j}^{\inf}$. This error vanishes if all functions $f_{i,j}$ share the same infimum $f_i^{\inf}$, in which case we recover the previous result for \ouralg using the local GD steps.
Utility guarantees for  \ouralg using the local GD steps can be found in~\Cref{corr:DP_PP_utility_local_IG}.

\section{Experiments} \label{sec:experiments}

We evaluate the performance of \ouralg~through numerical experiments on a non-convex optimization task involving deep neural network training. Following the experimental setup from prior work~\citep{shulgin2025smoothed} common for DP training, we use the CIFAR-10 dataset~\citep{krizhevsky2009learning} and the ResNet20 model~\citep{he2016deep}. Detailed settings and additional results are provided in the Appendix.
We analyze the performance of \ouralg in the differentially private setting by training the model for 300 communication rounds. We set the variance of added noise at $p \beta \sqrt{K \log(1/\delta)}\epsilon^{-1}$ for $\epsilon=8, \delta=10^{-5}$ and vary $\beta$ to simulate different privacy levels. The step size (learning rate) $\gamma$ is tuned for every combination of parameters $p$ and $\beta$.

\begin{figure*}[h]
    \centering
    \includegraphics[width=0.9\linewidth]{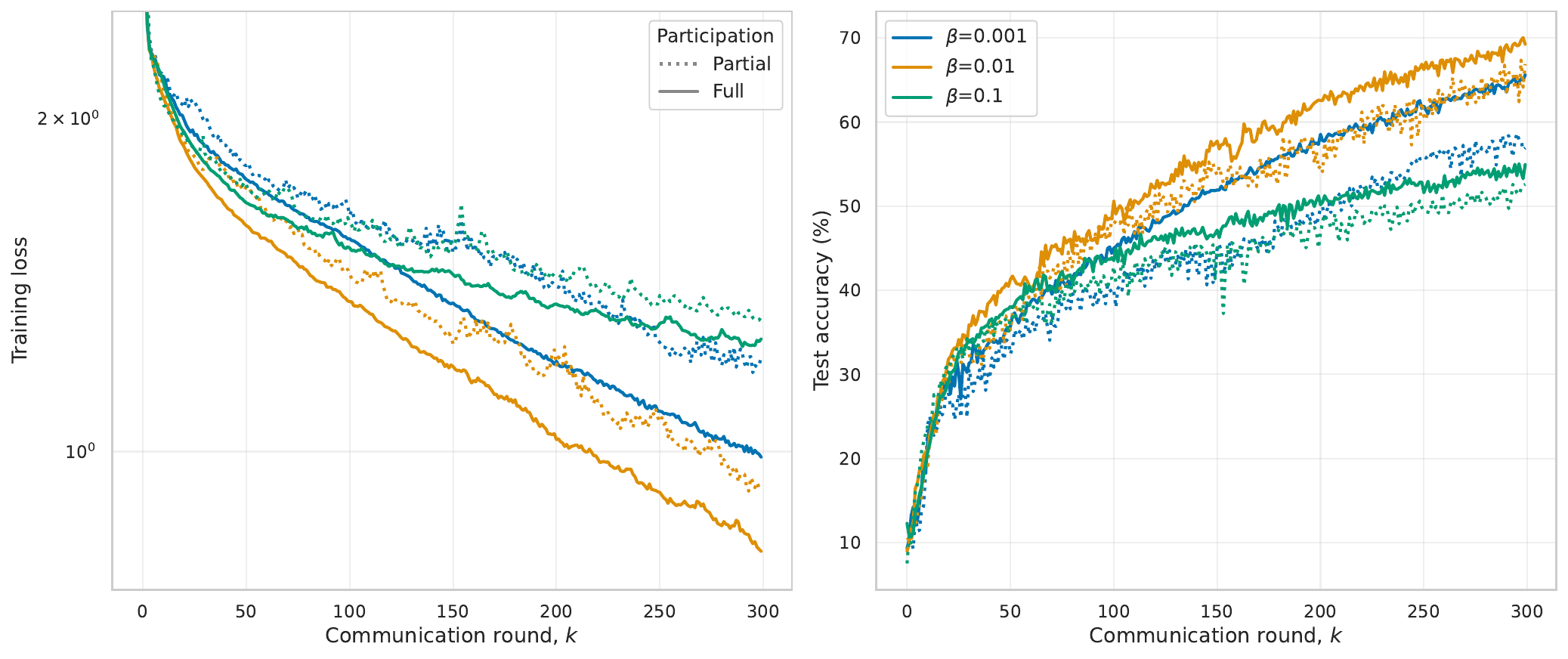}
    \caption{Convergence of \ouralg under Full [solid] and Partial participation [dotted] for $p=0.25$.}
    \label{fig:convergence}
\end{figure*}

The convergence behavior of \ouralg as a function of communication rounds is depicted in \Cref{fig:convergence}. The plots illustrate performance for full client participation ($p=1.0$, solid lines) and partial client participation ($p=0.25$, dotted lines) across three settings for the hyperparameter $\beta$.
The choice of $\beta$ markedly influences performance. Empirically, $\beta=0.01$ (orange lines) consistently delivers the best results, achieving the lowest training loss and highest test accuracy for both full and partial participation. For instance, with full participation, $\beta=0.01$ leads to approximately 70\% test accuracy, while $\beta=0.1$ (green lines) results in the poorest performance (around 55-60\% accuracy). Our theory (\Cref{thm:DP_PP_full}) supports this sensitivity, as $\beta$ influences both error feedback and the DP noise term (since $\sigma_{\text{DP}} \propto p\beta$). The convergence bound includes terms like $\sqrt{\beta^2 B(K+1)/M}$, implying an optimal $\beta$ balances error compensation and noise.

In terms of convergence with respect to the communication round, Full participation ($p=1.0$) outperforms Partial participation ($p=0.25$) for a fixed $\beta$. This is consistent with \Cref{thm:DP_PP_full}: the client sampling variance component of $B$ ($(p-1)^2/p$) is zero for $p=1$ but positive for $p=0.25$. Although the DP noise contribution to $B$ ($\sigma^2_{\rm DP}/p \propto p\beta^2$) is smaller for $p=0.25$, the client sampling variance appears more dominant in round-wise performance. These results underscore the trade-offs in selecting $\beta$ and the impact of client participation on round-wise performance.

\begin{figure*}[h] 
    \centering
    \includegraphics[width=0.9\linewidth]{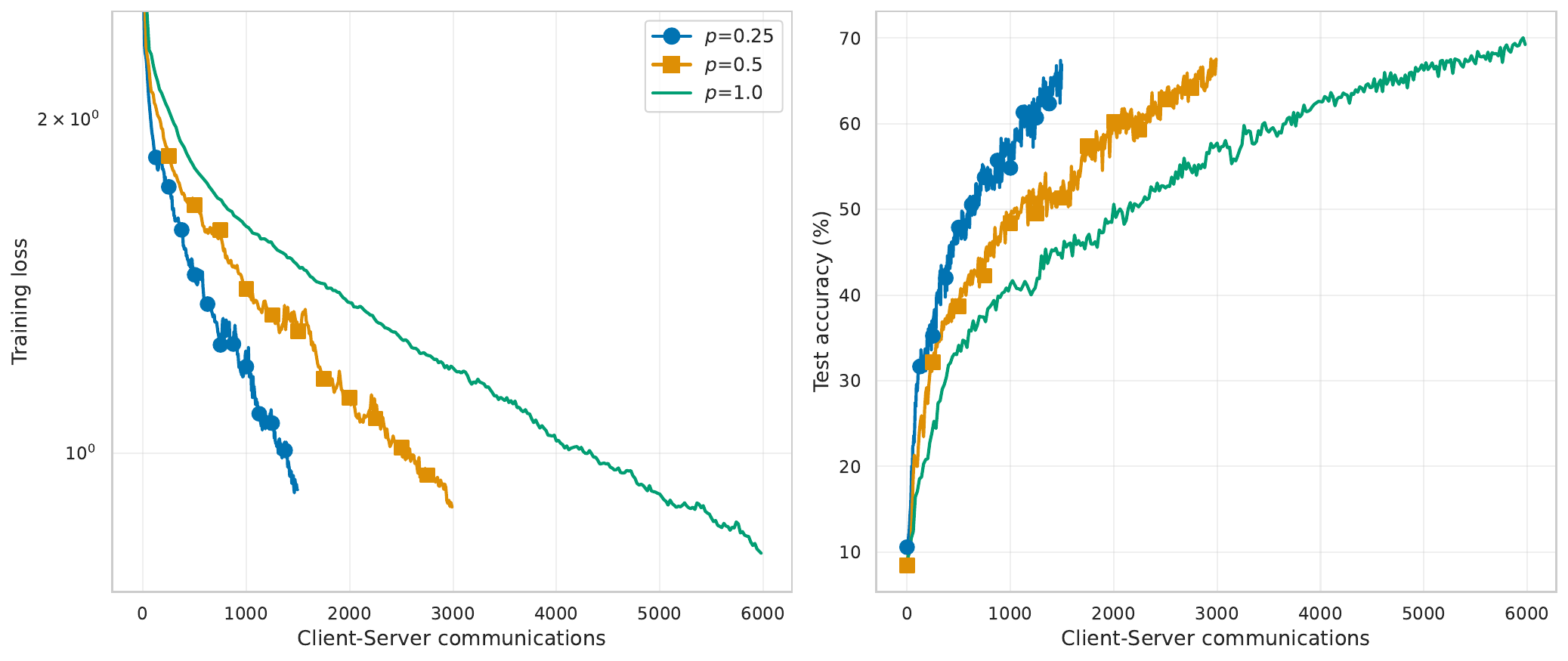}
    \caption{Convergence of \ouralg across different partial participation rates. Horizontal axis takes into account the total number of transmissions from client to server.}
    \label{fig:communication}
\end{figure*}

\Cref{fig:communication} further analyzes \ouralg's performance against the total number of client-server communications (i.e., $k \times p \times M$). This visualization offers direct insights into communication efficiency. Notably, configurations with smaller client participation probabilities ($p=0.25$ and $p=0.5$) achieve target performance levels with significantly fewer total client-server transmissions compared to full participation ($p=1.0$). For instance, to reach approximately 65\% test accuracy, $p=0.25$ (blue circles) requires about 1200 total communications, whereas $p=1.0$ (green line) needs nearly 4500.

This efficiency aligns with our theory (\Cref{corr:DP_PP_utility} and \Cref{thm:DP_PP_full}). With privacy amplification and $\sigma_{\rm DP} \propto p\beta$, the DP noise component in the error term $B$ ($2\sigma^2_{\rm DP}/p$) becomes proportional to $p\beta^2$. Thus, for a fixed $\beta$, a smaller $p$ reduces this noise contribution. This reduction, coupled with fewer transmissions per round for smaller $p$, explains the overall communication efficiency gains. Although full participation converges in fewer rounds (see \Cref{fig:convergence}), partial participation is more economical under a fixed total communication budget due to privacy amplification.

\section{Conclusion}

This paper presented \ouralg, the first differentially private federated learning algorithm to offer provable convergence for nonconvex, smooth problems without resorting to unrealistic assumptions such as bounded gradients or heterogeneity. \ouralg uniquely combines smoothed normalization and error compensation with essential practical FL components: local updates, distinct server/client learning rates, partial client participation (vital for privacy amplification), and DP noise. 
Finally, we verify the effectiveness of \ouralg by experiments on private deep neural network training.

\section*{Acknowledgements}

The research reported in this publication was supported by funding from King Abdullah
University of Science and Technology (KAUST): i) KAUST Baseline Research Scheme, ii)
Center of Excellence for Generative AI, under award number 5940, iii) SDAIA-KAUST Center
of Excellence in Artificial Intelligence and Data Science.

\bibliography{refs}

@inproceedings{mcmahan2017communication,
  title={Communication-efficient learning of deep networks from decentralized data},
  author={McMahan, Brendan and Moore, Eider and Ramage, Daniel and Hampson, Seth and y Arcas, Blaise Aguera},
  booktitle={Artificial intelligence and statistics},
  pages={1273--1282},
  year={2017},
  organization={PMLR}
}

@article{konevcny2016federated,
  title={Federated learning: Strategies for improving communication efficiency},
  author={Kone{\v{c}}n{\`y}, Jakub and McMahan, H Brendan and Yu, Felix X and Richt{\'a}rik, Peter and Suresh, Ananda Theertha and Bacon, Dave},
  journal={arXiv preprint arXiv:1610.05492},
  year={2016}
}

@article{caldas2018expanding,
  title={Expanding the reach of federated learning by reducing client resource requirements},
  author={Caldas, Sebastian and Kone{\v{c}}ny, Jakub and McMahan, H Brendan and Talwalkar, Ameet},
  journal={arXiv preprint arXiv:1812.07210},
  year={2018}
}

@inproceedings{boenisch2023curious,
  title={When the curious abandon honesty: Federated learning is not private},
  author={Boenisch, Franziska and Dziedzic, Adam and Schuster, Roei and Shamsabadi, Ali Shahin and Shumailov, Ilia and Papernot, Nicolas},
  booktitle={2023 IEEE 8th European Symposium on Security and Privacy (EuroS\&P)},
  pages={175--199},
  year={2023},
  organization={IEEE}
}

@inproceedings{
crawshaw2023episode,
title={{EPISODE}: Episodic Gradient Clipping with Periodic Resampled Corrections for Federated Learning with Heterogeneous Data},
author={Michael Crawshaw and Yajie Bao and Mingrui Liu},
booktitle={The Eleventh International Conference on Learning Representations },
year={2023}
}

@inproceedings{karimireddy2020scaffold,
  title={Scaffold: Stochastic controlled averaging for federated learning},
  author={Karimireddy, Sai Praneeth and Kale, Satyen and Mohri, Mehryar and Reddi, Sashank and Stich, Sebastian and Suresh, Ananda Theertha},
  booktitle={International conference on machine learning},
  pages={5132--5143},
  year={2020},
  organization={PMLR}
}

@INPROCEEDINGS{8835269,
  author={Melis, Luca and Song, Congzheng and De Cristofaro, Emiliano and Shmatikov, Vitaly},
  booktitle={2019 IEEE Symposium on Security and Privacy (SP)}, 
  title={Exploiting Unintended Feature Leakage in Collaborative Learning}, 
  year={2019},
  volume={},
  number={},
  pages={691-706}}

@article{dwork2014algorithmic,
  title={The algorithmic foundations of differential privacy},
  author={Dwork, Cynthia and Roth, Aaron and others},
  journal={Foundations and Trends{\textregistered} in Theoretical Computer Science},
  volume={9},
  number={3--4},
  pages={211--407},
  year={2014},
  publisher={Now Publishers, Inc.}
}

@inproceedings{zhang2022understanding,
  title={Understanding clipping for federated learning: Convergence and client-level differential privacy},
  author={Zhang, Xinwei and Chen, Xiangyi and Hong, Mingyi and Wu, Zhiwei Steven and Yi, Jinfeng},
  booktitle={International Conference on Machine Learning, ICML 2022},
  year={2022}
}

@article{li2022soteriafl,
  title={{SoteriaFL}: A unified framework for private federated learning with communication compression},
  author={Li, Zhize and Zhao, Haoyu and Li, Boyue and Chi, Yuejie},
  journal={Advances in Neural Information Processing Systems},
  volume={35},
  pages={4285--4300},
  year={2022}
}

@inproceedings{wang2023efficient,
  title={Efficient privacy-preserving stochastic nonconvex optimization},
  author={Wang, Lingxiao and Jayaraman, Bargav and Evans, David and Gu, Quanquan},
  booktitle={Uncertainty in Artificial Intelligence},
  pages={2203--2213},
  year={2023},
  organization={PMLR}
}

@inproceedings{lowy2023private,
  title={Private non-convex federated learning without a trusted server},
  author={Lowy, Andrew and Ghafelebashi, Ali and Razaviyayn, Meisam},
  booktitle={International Conference on Artificial Intelligence and Statistics},
  pages={5749--5786},
  year={2023},
  organization={PMLR}
}

@inproceedings{zhang2020private,
  title={Private and communication-efficient edge learning: A sparse differential Gaussian-masking distributed {SGD} approach},
  author={Zhang, Xin and Fang, Minghong and Liu, Jia and Zhu, Zhengyuan},
  booktitle={Proceedings of the Twenty-First International Symposium on Theory, Algorithmic Foundations, and Protocol Design for Mobile Networks and Mobile Computing},
  pages={261--270},
  year={2020}
}

@inproceedings{li2024improved,
  title={An improved analysis of per-sample and per-update clipping in federated learning},
  author={Li, Bo and Jiang, Xiaowen and Schmidt, Mikkel N and Alstr{\o}m, Tommy Sonne and Stich, Sebastian U},
  booktitle={The Twelfth International Conference on Learning Representations},
  year={2024}
}

@inproceedings{mishchenko2022proxskip,
  title={Proxskip: Yes! local gradient steps provably lead to communication acceleration! finally!},
  author={Mishchenko, Konstantin and Malinovsky, Grigory and Stich, Sebastian and Richt{\'a}rik, Peter},
  booktitle={International Conference on Machine Learning},
  pages={15750--15769},
  year={2022},
  organization={PMLR}
}

@article{wang2021field,
  title={A field guide to federated optimization},
  author={Wang, Jianyu and Charles, Zachary and Xu, Zheng and Joshi, Gauri and McMahan, H Brendan and Al-Shedivat, Maruan and Andrew, Galen and Avestimehr, Salman and Daly, Katharine and Data, Deepesh and others},
  journal={arXiv preprint arXiv:2107.06917},
  year={2021}
}

@article{kairouz2021advances,
  title={Advances and open problems in federated learning},
  author={Kairouz, Peter and McMahan, H Brendan and Avent, Brendan and Bellet, Aur{\'e}lien and Bennis, Mehdi and Bhagoji, Arjun Nitin and Bonawitz, Kallista and Charles, Zachary and Cormode, Graham and Cummings, Rachel and others},
  journal={Foundations and trends{\textregistered} in machine learning},
  volume={14},
  number={1--2},
  pages={1--210},
  year={2021},
  publisher={Now Publishers, Inc.}
}

@article{chen2020optimal,
  title={Optimal client sampling for federated learning},
  author={Chen, Wenlin and Horvath, Samuel and Richtarik, Peter},
  journal={arXiv preprint arXiv:2010.13723},
  year={2020}
}

@article{shulgin2024convergence,
  title={On the convergence of {DP-SGD} with adaptive clipping},
  author={Shulgin, Egor and Richt{\'a}rik, Peter},
  journal={arXiv preprint arXiv:2412.19916},
  year={2024}
}

@article{bu2023automatic,
  title={Automatic clipping: Differentially private deep learning made easier and stronger},
  author={Bu, Zhiqi and Wang, Yu-Xiang and Zha, Sheng and Karypis, George},
  journal={Advances in Neural Information Processing Systems},
  volume={36},
  pages={41727--41764},
  year={2023}
}

@inproceedings{noble2022differentially,
  title={Differentially private federated learning on heterogeneous data},
  author={Noble, Maxence and Bellet, Aur{\'e}lien and Dieuleveut, Aymeric},
  booktitle={International conference on artificial intelligence and statistics},
  pages={10110--10145},
  year={2022},
  organization={PMLR}
}

@inproceedings{safran2020good,
  title={How good is {SGD} with random shuffling?},
  author={Safran, Itay and Shamir, Ohad},
  booktitle={Conference on Learning Theory},
  pages={3250--3284},
  year={2020},
  organization={PMLR}
}

@article{bertsekas2011incremental,
  title={Incremental gradient, subgradient, and proximal methods for convex optimization: A survey},
  author={Bertsekas, Dimitri P and others},
  journal={Optimization for Machine Learning},
  volume={2010},
  number={1-38},
  pages={3},
  year={2011}
}

@inproceedings{
das2021convergence,
title={Differentially Private Federated Learning with Normalized Updates},
author={Rudrajit Das and Abolfazl Hashemi and sujay sanghavi and Inderjit S Dhillon},
booktitle={OPT 2022: Optimization for Machine Learning (NeurIPS 2022 Workshop)},
year={2022}
}

@article{khirirat2023clip21,
  title={Clip21: Error feedback for gradient clipping},
  author={Khirirat, Sarit and Gorbunov, Eduard and Horv{\'a}th, Samuel and Islamov, Rustem and Karray, Fakhri and Richt{\'a}rik, Peter},
  journal={arXiv preprint arXiv:2305.18929},
  year={2023}
}

@article{shulgin2025smoothed,
  title={Smoothed Normalization for Efficient Distributed Private Optimization},
  author={Shulgin, Egor and Khirirat, Sarit and Richt{\'a}rik, Peter},
  journal={arXiv preprint arXiv:2502.13482},
  year={2025}
}

@inproceedings{seide20141,
  title={1-bit stochastic gradient descent and its application to data-parallel distributed training of speech {DNN}s.},
  author={Seide, Frank and Fu, Hao and Droppo, Jasha and Li, Gang and Yu, Dong},
  booktitle={Interspeech},
  volume={2014},
  pages={1058--1062},
  year={2014},
  organization={Singapore}
}

@article{richtarik2021ef21,
  title={{EF21:} A new, simpler, theoretically better, and practically faster error feedback},
  author={Richt{\'a}rik, Peter and Sokolov, Igor and Fatkhullin, Ilyas},
  journal={Advances in Neural Information Processing Systems},
  volume={34},
  pages={4384--4396},
  year={2021}
}

@inproceedings{
gao2023econtrol,
title={{EC}ontrol: Fast Distributed Optimization with Compression and Error Control},
author={Yuan Gao and Rustem Islamov and Sebastian U Stich},
booktitle={The Twelfth International Conference on Learning Representations},
year={2024}
}

@inproceedings{bao2025efskip,
  title={{EFSkip}: A New Error Feedback with Linear Speedup for Compressed Federated Learning with Arbitrary Data Heterogeneity},
  author={Bao, Hongyan and Chen, Pengwen and Sun, Ying and Li, Zhize},
  booktitle={Proceedings of the AAAI Conference on Artificial Intelligence},
  volume={39},
  pages={15489--15497},
  year={2025}
}

@article{fatkhullin2023momentum,
  title={Momentum provably improves error feedback!},
  author={Fatkhullin, Ilyas and Tyurin, Alexander and Richt{\'a}rik, Peter},
  journal={Advances in Neural Information Processing Systems},
  volume={36},
  pages={76444--76495},
  year={2023}
}

@article{liu2022communication,
  title={A communication-efficient distributed gradient clipping algorithm for training deep neural networks},
  author={Liu, Mingrui and Zhuang, Zhenxun and Lei, Yunwen and Liao, Chunyang},
  journal={Advances in Neural Information Processing Systems},
  volume={35},
  pages={26204--26217},
  year={2022}
}

@inproceedings{abadi2016deep,
  title={Deep learning with differential privacy},
  author={Abadi, Martin and Chu, Andy and Goodfellow, Ian and McMahan, H Brendan and Mironov, Ilya and Talwar, Kunal and Zhang, Li},
  booktitle={Proceedings of the 2016 ACM SIGSAC conference on computer and communications security},
  pages={308--318},
  year={2016}
}

@article{geyer2017differentially,
  title={Differentially private federated learning: A client level perspective},
  author={Geyer, Robin C and Klein, Tassilo and Nabi, Moin},
  journal={arXiv preprint arXiv:1712.07557},
  year={2017}
}

@inproceedings{triastcyn2019federated,
  title={Federated learning with bayesian differential privacy},
  author={Triastcyn, Aleksei and Faltings, Boi},
  booktitle={2019 IEEE International Conference on Big Data (Big Data)},
  pages={2587--2596},
  year={2019},
  organization={IEEE}
}

@inproceedings{mishchenko2022proximal,
  title={Proximal and federated random reshuffling},
  author={Mishchenko, Konstantin and Khaled, Ahmed and Richt{\'a}rik, Peter},
  booktitle={International Conference on Machine Learning},
  pages={15718--15749},
  year={2022},
  organization={PMLR}
}

@article{sadiev2022federated,
  title={Federated optimization algorithms with random reshuffling and gradient compression},
  author={Sadiev, Abdurakhmon and Malinovsky, Grigory and Gorbunov, Eduard and Sokolov, Igor and Khaled, Ahmed and Burlachenko, Konstantin and Richt{\'a}rik, Peter},
  journal={arXiv preprint arXiv:2206.07021},
  year={2022}
}

@article{malinovsky2022federated,
  title={Federated random reshuffling with compression and variance reduction},
  author={Malinovsky, Grigory and Richt{\'a}rik, Peter},
  journal={arXiv preprint arXiv:2205.03914},
  year={2022}
}

@article{mishchenko2020random,
  title={Random reshuffling: Simple analysis with vast improvements},
  author={Mishchenko, Konstantin and Khaled, Ahmed and Richt{\'a}rik, Peter},
  journal={Advances in Neural Information Processing Systems},
  volume={33},
  pages={17309--17320},
  year={2020}
}

@inproceedings{haochen2019random,
  title={Random shuffling beats {SGD} after finite epochs},
  author={Haochen, Jeff and Sra, Suvrit},
  booktitle={International Conference on Machine Learning},
  pages={2624--2633},
  year={2019},
  organization={PMLR}
}

@article{safran2021random,
  title={Random shuffling beats {SGD} only after many epochs on ill-conditioned problems},
  author={Safran, Itay and Shamir, Ohad},
  journal={Advances in Neural Information Processing Systems},
  volume={34},
  pages={15151--15161},
  year={2021}
}

@InProceedings{yun2021can,
  title = 	 {Open Problem: Can Single-Shuffle {SGD} be Better than Reshuffling {SGD} and {GD}?},
  author =       {Yun, Chulhee and Sra, Suvrit and Jadbabaie, Ali},
  booktitle = 	 {Proceedings of Thirty Fourth Conference on Learning Theory},
  pages = 	 {4653--4658},
  year = 	 {2021},
  volume = 	 {134},
  series = 	 {Proceedings of Machine Learning Research},
  month = 	 {15--19 Aug},
  publisher =    {PMLR}
}

@inproceedings{mcmahan2018learning,
  title={Learning Differentially Private Recurrent Language Models},
  author={McMahan, H Brendan and Ramage, Daniel and Talwar, Kunal and Zhang, Li},
  booktitle={International Conference on Learning Representations},
  year={2018}
}

@inproceedings{haddadpour2021federated,
  title={Federated learning with compression: Unified analysis and sharp guarantees},
  author={Haddadpour, Farzin and Kamani, Mohammad Mahdi and Mokhtari, Aryan and Mahdavi, Mehrdad},
  booktitle={International Conference on Artificial Intelligence and Statistics},
  pages={2350--2358},
  year={2021},
  organization={PMLR}
}

@inproceedings{gorbunov2021local,
  title={Local {SGD}: Unified theory and new efficient methods},
  author={Gorbunov, Eduard and Hanzely, Filip and Richt{\'a}rik, Peter},
  booktitle={International Conference on Artificial Intelligence and Statistics},
  pages={3556--3564},
  year={2021},
  organization={PMLR}
}

@inproceedings{patel2024limits,
  title={The limits and potentials of local {SGD} for distributed heterogeneous learning with intermittent communication},
  author={Patel, Kumar Kshitij and Glasgow, Margalit and Zindari, Ali and Wang, Lingxiao and Stich, Sebastian U and Cheng, Ziheng and Joshi, Nirmit and Srebro, Nathan},
  booktitle={The Thirty Seventh Annual Conference on Learning Theory},
  pages={4115--4157},
  year={2024},
  organization={PMLR}
}

@article{koloskova2023convergence,
  title={On convergence of incremental gradient for non-convex smooth functions},
  author={Koloskova, Anastasia and Doikov, Nikita and Stich, Sebastian U and Jaggi, Martin},
  journal={arXiv preprint arXiv:2305.19259},
  year={2023}
}

@inproceedings{malinovsky2023server,
  title={Server-side stepsizes and sampling without replacement provably help in federated optimization},
  author={Malinovsky, Grigory and Mishchenko, Konstantin and Richt{\'a}rik, Peter},
  booktitle={Proceedings of the 4th International Workshop on Distributed Machine Learning},
  pages={85--104},
  year={2023}
}

@inproceedings{
li2024power,
title={The Power of Extrapolation in Federated Learning},
author={Hanmin Li and Kirill Acharya and Peter Richt{\'a}rik},
booktitle={The Thirty-eighth Annual Conference on Neural Information Processing Systems},
year={2024}
}

@article{reddi2020adaptive,
  title={Adaptive federated optimization},
  author={Reddi, Sashank and Charles, Zachary and Zaheer, Manzil and Garrett, Zachary and Rush, Keith and Kone{\v{c}}n{\`y}, Jakub and Kumar, Sanjiv and McMahan, H Brendan},
  journal={arXiv preprint arXiv:2003.00295},
  year={2020}
}

@article{malinovsky2023federated,
  title={Federated learning with regularized client participation},
  author={Malinovsky, Grigory and Horv{\'a}th, Samuel and Burlachenko, Konstantin and Richt{\'a}rik, Peter},
  journal={arXiv preprint arXiv:2302.03662},
  year={2023}
}

@article{charles2020outsized,
  title={On the outsized importance of learning rates in local update methods},
  author={Charles, Zachary and Kone{\v{c}}n{\`y}, Jakub},
  journal={arXiv preprint arXiv:2007.00878},
  year={2020}
}

@inproceedings{koloskova2020unified,
  title={A unified theory of decentralized {SGD} with changing topology and local updates},
  author={Koloskova, Anastasia and Loizou, Nicolas and Boreiri, Sadra and Jaggi, Martin and Stich, Sebastian},
  booktitle={International conference on machine learning},
  pages={5381--5393},
  year={2020},
  organization={PMLR}
}

@inproceedings{malinovskiy2020local,
  title={From local {SGD} to local fixed-point methods for federated learning},
  author={Malinovskiy, Grigory and Kovalev, Dmitry and Gasanov, Elnur and Condat, Laurent and Richtarik, Peter},
  booktitle={International Conference on Machine Learning},
  pages={6692--6701},
  year={2020},
  organization={PMLR}
}

@inproceedings{khaled2020tighter,
  title={Tighter theory for local {SGD} on identical and heterogeneous data},
  author={Khaled, Ahmed and Mishchenko, Konstantin and Richt{\'a}rik, Peter},
  booktitle={International conference on artificial intelligence and statistics},
  pages={4519--4529},
  year={2020},
  organization={PMLR}
}

@article{khaled2019gradient,
  title={Gradient descent with compressed iterates},
  author={Khaled, Ahmed and Richt{\'a}rik, Peter},
  journal={arXiv preprint arXiv:1909.04716},
  year={2019}
}

@inproceedings{shulgin2022shifted,
  title={Shifted compression framework: Generalizations and improvements},
  author={Shulgin, Egor and Richt{\'a}rik, Peter},
  booktitle={Uncertainty in Artificial Intelligence},
  pages={1813--1823},
  year={2022},
  organization={PMLR}
}

@article{chraibi2019distributed,
  title={Distributed fixed point methods with compressed iterates},
  author={Chraibi, S{\'e}lim and Khaled, Ahmed and Kovalev, Dmitry and Richt{\'a}rik, Peter and Salim, Adil and Tak{\'a}{\v{c}}, Martin},
  journal={arXiv preprint arXiv:1912.09925},
  year={2019}
}

@inproceedings{gruntkowska2023ef21,
  title={{EF21-P} and friends: Improved theoretical communication complexity for distributed optimization with bidirectional compression},
  author={Gruntkowska, Kaja and Tyurin, Alexander and Richt{\'a}rik, Peter},
  booktitle={International Conference on Machine Learning},
  pages={11761--11807},
  year={2023},
  organization={PMLR}
}

@article{islamov2025double,
  title={Double Momentum and Error Feedback for Clipping with Fast Rates and Differential Privacy},
  author={Islamov, Rustem and Horvath, Samuel and Lucchi, Aurelien and Richtarik, Peter and Gorbunov, Eduard},
  journal={arXiv preprint arXiv:2502.11682},
  year={2025}
}

@article{stich2019error,
  author  = {Sebastian U. Stich and Sai Praneeth Karimireddy},
  title   = {The Error-Feedback framework: {SGD} with Delayed Gradients},
  journal = {Journal of Machine Learning Research},
  year    = {2020},
  volume  = {21},
  number  = {237},
  pages   = {1--36}
}

@article{gorbunov2020linearly,
  title={Linearly converging error compensated {SGD}},
  author={Gorbunov, Eduard and Kovalev, Dmitry and Makarenko, Dmitry and Richt{\'a}rik, Peter},
  journal={Advances in Neural Information Processing Systems},
  volume={33},
  pages={20889--20900},
  year={2020}
}

@inproceedings{he2016deep,
	title={Deep residual learning for image recognition},
	author={He, Kaiming and Zhang, Xiangyu and Ren, Shaoqing and Sun, Jian},
	booktitle={Proceedings of the IEEE Conference on Computer Vision and Pattern Recognition (CVPR)},
	pages={770--778},
	year={2016}
}

@techreport{krizhevsky2009learning,
	title={Learning multiple layers of features from tiny images},
	author={Krizhevsky, Alex and Hinton, Geoffrey and others},
	year={2009},
	jnumber = {Technical Report TR-2009},
	institution = {University of Toronto,  Toronto}
}

@misc{idelbayev18a,
  author= "Yerlan Idelbayev",
  title="Proper {ResNet} Implementation for {CIFAR10/CIFAR100} in {PyTorch}",
  howpublished="\url{https://github.com/akamaster/pytorch_resnet_cifar10}",
  note="Accessed: 2024-12-31"
}

\appendix
\tableofcontents

\section{Additional experiments and details}\label{app:exp}

\paragraph{Additional details.}
All methods are run using a constant learning rate, without auxiliary techniques such as learning rate schedules, warm-up phases, or weight decay. The CIFAR-10 dataset is partitioned into 90\% for training and 10\% for testing. Training samples are randomly shuffled and evenly distributed across $n=20$ workers, each using a local batch size of 32. We use
a fixed random seed (42) to ensure reproducibility.
Our implementation builds upon the publicly available GitHub repository of \citet{idelbayev18a}, and all experiments are conducted on a single NVIDIA GeForce RTX 3090 GPU.

We use a fixed smoothed normalization parameter $\alpha=0.01$, as it was shown to have an insignificant effect on convergence~\citep{shulgin2025smoothed}. Server normalization (Line 12 in Algorithm~\ref{alg:norm_norm21_dp_v1-fed}) is not used, as omitting it empirically improves final performance~\citep{shulgin2025smoothed}. All methods are evaluated across the following hyperparameter combinations: step size $\gamma \in \{0.001, 0.01, 0.1\}$ and sensitivity threshold $\beta \in \{0.001, 0.01, 0.1\}$.

\subsection{\ouralg vs \algname{FedAvg}}

We compare the performance of our Algorithm~\ref{alg:norm_norm21_dp_v1-fed} (\ouralg) with the standard \algname{FedAvg} approach, as defined in Section~\ref{sec:preliminaries}:
\begin{eqnarray*}
    x^{k+1} = x^k - \frac{\eta}{B} \left[\sum_{i\in S^k } \Psi(x^k - \mathcal{T}_i(x^k)) + z_i^k \right],
\end{eqnarray*}
where $\Psi$ is the smoothed normalization operator, $\mathcal{T}_i(x) = x - \gamma \nabla f_i(x)$ is the local gradient mapping, $\eta = \gamma$, and $p=1$ in the Differentially Private (DP) setting.
We follow the same experimental setup as described in Section~\ref{sec:experiments}.

Figure~\ref{fig:avg_vs_normec} presents the convergence of training loss and test accuracy for both methods across different values of the sensitivity parameter $\beta$. The results demonstrate that the Error Compensation (EC) mechanism in \ouralg~consistently accelerates convergence and improves test accuracy compared to \algname{FedAvg}, across all privacy levels (i.e., all tested values of $\beta$). Notably, \ouralg~achieves its best performance for $\beta=0.01$, which aligns with the findings in Section~\ref{sec:experiments}.
\begin{figure*}[!htbp]
    \centering
    \includegraphics[width=0.9\linewidth]{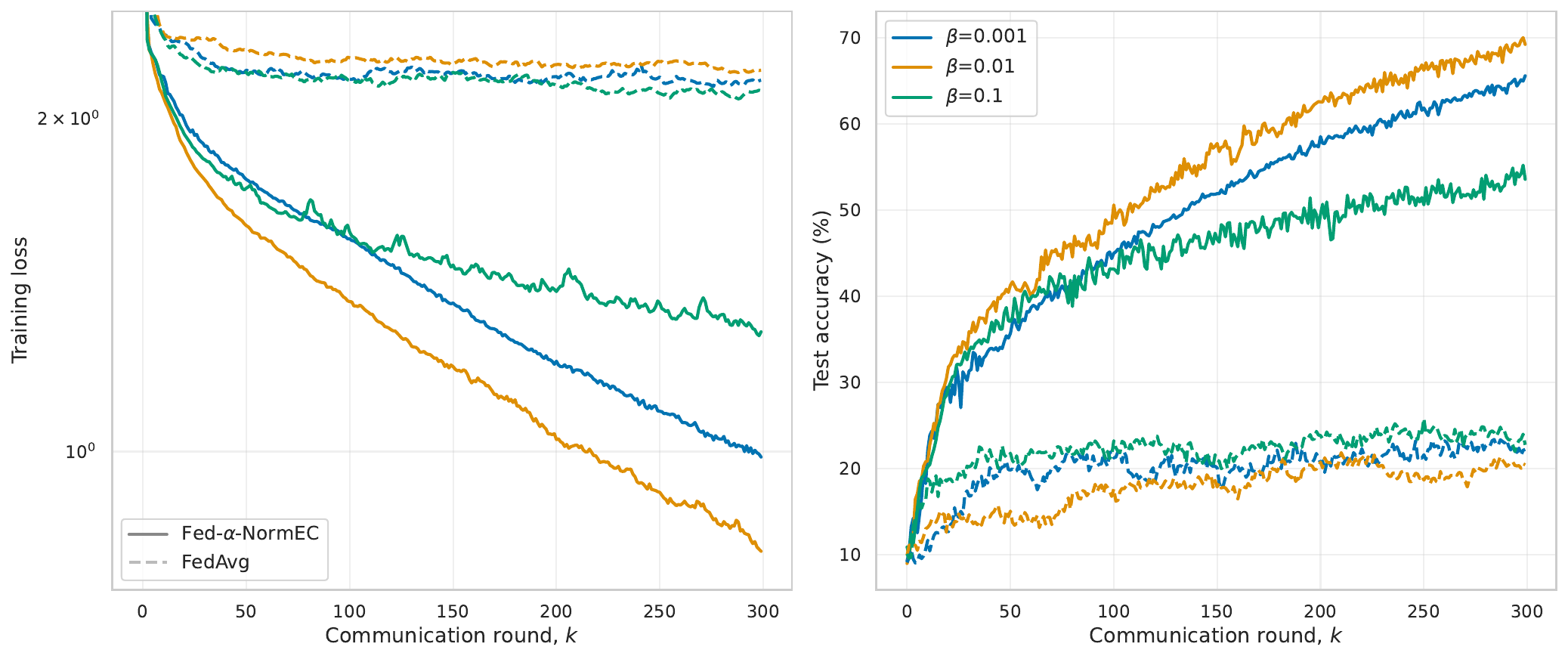}
    \caption{Error Compensation (EC) provides significant benefits across various $\beta$ values.}
\vskip -0.1in
    \label{fig:avg_vs_normec}
\end{figure*}

To further analyze the effect of hyperparameters, Figure~\ref{fig:avg_heatmap} shows the highest test accuracy achieved by \algname{FedAvg} for each $(\beta, \gamma)$ pair. The optimal performance for \algname{FedAvg} is observed at $\beta=0.1$, while the best results are generally found along the diagonal, where the product $\beta \cdot \gamma = 0.001$.
\begin{figure*}[h]
    \centering
    \includegraphics[width=0.4\linewidth]{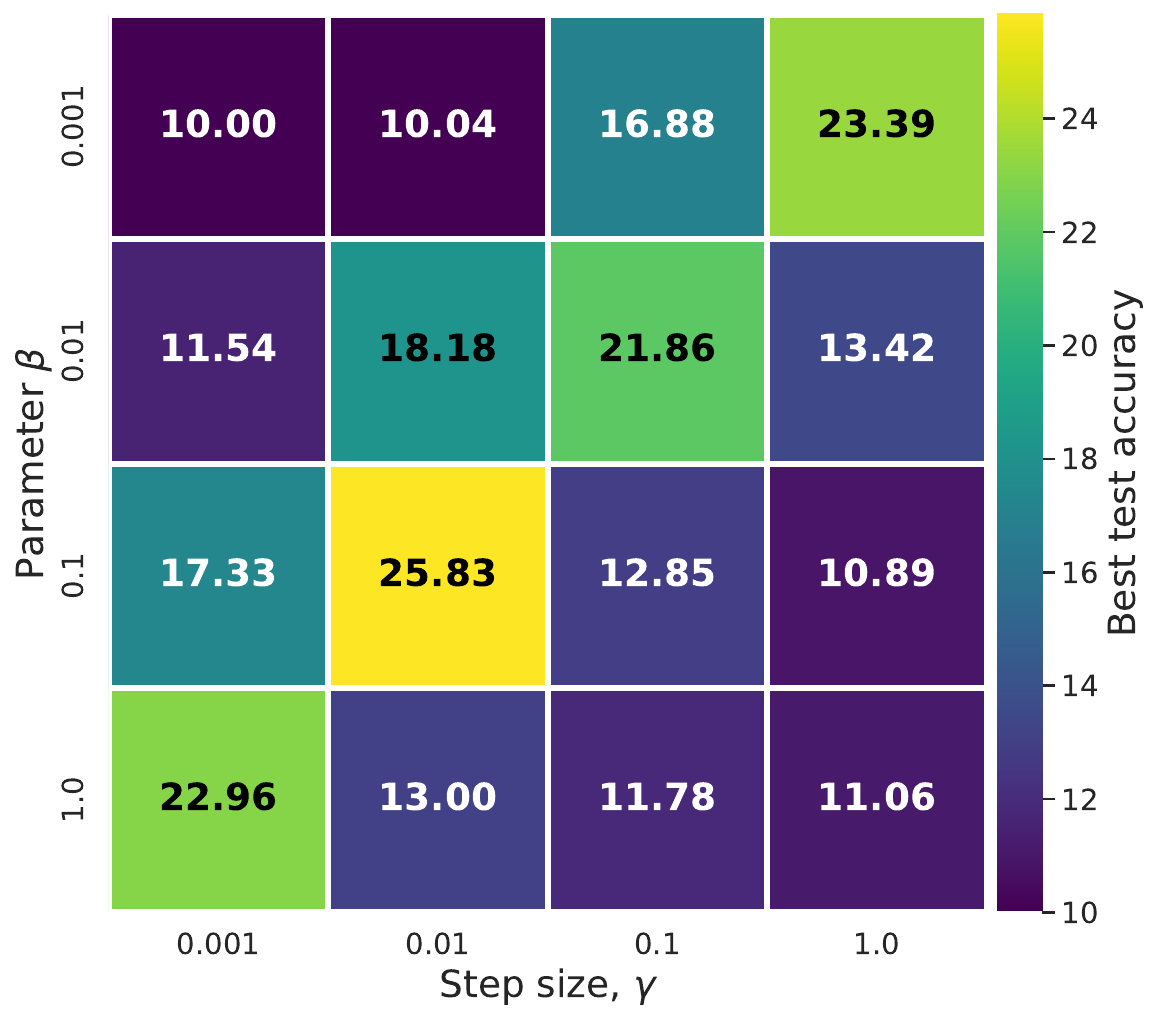}
    \caption{The highest test accuracy achieved by \algname{FedAvg} for different $\beta$ and $\gamma$ parameters.}
    \label{fig:avg_heatmap}
\vskip -0.1in
\end{figure*}

Importantly, prior works, including \cite{khirirat2023clip21, shulgin2025smoothed}, have shown that \algname{FedAvg} with clipping or normalization may fail to converge in certain settings, whereas \ouralg~remains robust and convergent. Our results further support this observation, highlighting the effectiveness of the Error Compensation mechanism in improving both convergence speed and final accuracy, especially in privacy-constrained federated learning scenarios.

\section{Useful Lemmas}

We introduce five lemmas that are useful for our  analysis.

First, \Cref{lemma:bound_v_and_T} establishes the bounds for  $\norm{\frac{x^k-\mathcal{T}_i(x^k)}{\gamma}-v_i^{k}}$ and $\norm{\frac{x^k - \mathcal{T}_i(x^k)}{\gamma} - v_i^{k+1}}$.
Both quantities will be applied in the induction proof, which is the first step of the convergence proof of~\ouralg.

\begin{lemma}\label{lemma:bound_v_and_T}
Let $v_i^{k}\in\R^d$ be governed by 
\begin{eqnarray*}
	v_i^{k+1} = v_i^k + \beta \Normalize{\alpha}{ \frac{x^k - \mathcal{T}_i(x^{k})}{\gamma} - v_i^{k}},  \text{ for $i \in [1,M]$ and }  k \geq 0,
\end{eqnarray*}	
and let the fixed-point operator $\mathcal{T}_i(\cdot)$ satisfy 
\begin{eqnarray*}
	\norm{\mathcal{T}_i(x)-\mathcal{T}_i(y)} \leq \rho \norm{x-y},  \text{ for $\rho > 0$ and } x,y\in\R^d.
\end{eqnarray*}	
If $\norm{\frac{x^k - \mathcal{T}_i(x^{k})}{\gamma} - v_i^{k}} \leq C$ for some $C>0$, $\norm{x^{k+1}-x^k} \leq \eta$, $\frac{\beta}{\alpha + C} < 1$, and $\eta \leq \frac{\gamma\beta C}{(1+\rho) (\alpha + C)}$, then $\norm{\frac{x^{k+1} - \mathcal{T}_i(x^{k+1})}{\gamma} - v_i^{k+1}} \leq C$ and $\norm{\frac{x^{k} - \mathcal{T}_i(x^{k})}{\gamma} - v_i^{k+1}} \leq C$.	
\end{lemma}
\begin{proof}
From the definition of the Euclidean norm, 
\begin{eqnarray*}
    \norm{ P_i(x^{k+1}) - v_i^{k+1} }
    & \overset{v_i^{k+1}}{=} & \norm{P_i(x^{k+1}) - v_i^k - \beta N_\alpha(P_i(x^{k})-v_i^k)} \\
    & \overset{\text{triangle inequality}}{\leq}& \norm{P_i(x^{k+1})-P_i(x^k)} + \norm{P_i(x^k)-v_i^k - \Normalize{\alpha}{P_i(x^k)-v_i^k}},
\end{eqnarray*}
where $P_i(x) = (x-\mathcal{T}_i(x))/\gamma$.

Next, we bound $\norm{P_i(x^{k+1})-P_i(x^k)}$.
From the definition of $P_i(x)$, 
\begin{align*}
    \norm{P_i(x^{k+1})-P_i(x^k)}
    & = \left\| \frac{x^{k+1} - \mathcal{T}_i(x^{k+1})}{\gamma} - \frac{x^{k} - \mathcal{T}_i(x^{k})}{\gamma} \right\| \\
    &\overset{(*)}{\leq} \frac{1}{\gamma}\left( \left\| x^{k+1} - x^k \right\| + \left\| \mathcal{T}_i(x^k) - \mathcal{T}_i(x^{k+1})  \right\| \right)\\
    &\overset{(**)}{\leq} \frac{1}{\gamma}\left(1+\rho\right) \left\|x^{k+1} - x^k\right\|.
\end{align*}
Here, we reach $(*)$ by the triangle inequality,  and  $(**)$ by the fact that $\norm{\mathcal{T}_i(x)-\mathcal{T}_i(y)} \leq \rho \norm{x-y}$ for $\rho >0$ and $x,y\in\R^d$. 
Plugging the upper-bound for $\norm{P_i(x^{k+1})-P_i(x^k)}$ into the main inequality yields 
\begin{eqnarray*}
    \norm{ P_i(x^{k+1}) - v_i^{k+1} }
    \leq \frac{1}{\gamma}(1+\rho)\norm{x^{k+1}-x^k} + \norm{P_i(x^k)-v_i^k - \Normalize{\alpha}{P_i(x^k)-v_i^k}}.
\end{eqnarray*}

Next, we bound $\norm{P_i(x^k)-v_i^k - \Normalize{\alpha}{P_i(x^k)-v_i^k}}$. From the property of $\Normalize{\alpha}(\cdot)$ according to  Lemma 1 of~\cite{shulgin2025smoothed}, 
\begin{align*}
\norm{P_i(x^k)-v_i^k - \Normalize{\alpha}{P_i(x^k)-v_i^k}}
& \leq \left\vert 1 - \frac{\beta}{\alpha + \norm{\frac{x^k - \mathcal{T}_i(x^k)}{\gamma} -v_i^k} } \right\vert \norm{ \frac{x^k - \mathcal{T}_i(x^k)}{\gamma} -v_i^k } .
\end{align*}

If  $\norm{ \frac{x^k - \mathcal{T}_i(x^k)}{\gamma} -v_i^k} \leq C$ for some $C>0$, and $\frac{\beta}{\alpha+C}< 1$, then 
\begin{align*}
\norm{P_i(x^k)-v_i^k - \Normalize{\alpha}{P_i(x^k)-v_i^k}}
& \leq \left\vert 1 - \frac{\beta}{\alpha + C} \right\vert C \\
& \leq \left(1 - \frac{\beta}{\alpha + C} \right)C.
\end{align*}
Hence, by plugging the upper-bound for $\norm{P_i(x^k)-v_i^k - \Normalize{\alpha}{P_i(x^k)-v_i^k}}$ into the main inequality, we obtain
\begin{eqnarray*}
    \norm{ P_i(x^{k+1}) - v_i^{k+1} }
    \leq \frac{1}{\gamma}(1+\rho)\norm{x^{k+1}-x^k} + \left(1 - \frac{\beta}{\alpha + C} \right)C.
\end{eqnarray*}

If  $\left\| x^{k+1} - x^k \right\| \leq \eta $, then 
\begin{eqnarray*}
    \norm{ P_i(x^{k+1}) - v_i^{k+1} }
    \leq \frac{1}{\gamma}(1+\rho)\eta + \left(1 - \frac{\beta}{\alpha + C} \right)C.
\end{eqnarray*}

If $\eta \leq \frac{\gamma}{1+\rho}\frac{\beta C}{ (\alpha +C)}$, then we can show that 
$$\norm{P_i(x^{k+1}) - v_i^{k+1}} \leq C,$$
and that 
\begin{eqnarray*}
	    \norm{ P_i(x^k) - v_i^{k+1}}     	& \overset{v_i^{k+1}}{=} & \norm{ P_i(x^k) - v_i^k -  \beta \Normalize{\alpha}{ P_i(x^k) - v_i^{k}}}  \\ 
	& \overset{\text{Lemma 1 of~\cite{shulgin2025smoothed}}}{\leq} & \left\vert  1 - \frac{\beta}{\alpha + \norm{\frac{x^{k} - \mathcal{T}_i(x^{k})}{\gamma} - v_i^{k}}} \right\vert \norm{\frac{x^{k} - \mathcal{T}_i(x^{k})}{\gamma} - v_i^{k}} \\
	& \overset{\beta/(\alpha+C)<1}{\leq} & \left( 1 - \frac{\beta}{\alpha + C } \right)  C \\
	& \leq & C.
\end{eqnarray*}	

\end{proof}

Next, \Cref{lemma:bound_noise} provides the upper-bound for $\norm{e^k}$, where $e^{k+1}=e^k + \beta \frac{1}{M}\sum_{i=1}^M z_i^k$, and $z_i^k$ is the random vector with zero mean and $\sigma^2$-variance.

\begin{lemma}\label{lemma:bound_noise}
Let $e^k \in \R^d$ be governed by 	
\begin{eqnarray*}
	e^{k+1}  = e^k + \beta z^k, \quad \text{for $0 \leq k \leq K$},
\end{eqnarray*}	
where $z^k = \frac{1}{M}\sum_{i=1}^M z_i^k$ and each $z_i^k \in \R^d$ is an independent random vector satisfying
\begin{eqnarray*}
 \Exp{z_i^k} = 0, \quad \text{and} \quad \Exp{\sqnorm{z_i^k}} \leq \sigma^2.
\end{eqnarray*}
Then, 
\begin{eqnarray*}
	\Exp{\norm{e^{k+1}}} \leq \Exp{\norm{e^0}} + \sqrt{\frac{\beta^2 (K+1)\sigma^2}{M}}.
\end{eqnarray*}	
\end{lemma}
\begin{proof}
By applying the recursion of $e^{k+1}$ recursively, 	
\begin{eqnarray*}
	e^{k+1} = e^0 + \beta \sum_{l=0}^{k} z^l.
\end{eqnarray*}	
By taking the Euclidean norm, by the triangle inequality, and by taking the expectation, 
\begin{eqnarray*}
	\Exp{\norm{e^{k+1}}} \leq  \Exp{\norm{e^0}} + \Exp{ \norm{ \beta \sum_{l=0}^{k} z^l } } .
\end{eqnarray*}	
By Jensen's inequality, 
\begin{eqnarray*}
	\Exp{\norm{e^{k+1}}} 
	& \leq &  \Exp{\norm{e^0}} + \sqrt{ \Exp{ \sqnorm{ \beta \sum_{l=0}^{k} z^l } } } \\
	& = & \Exp{\norm{e^0}} + \sqrt{ \beta^2 \sum_{l=0}^k \Exp{\sqnorm{z^l}}  + \beta^2 \sum_{i \neq j} \Exp{\inp{z^i}{z^j}} }. 
\end{eqnarray*}	
Since $z_i^k$ is independent of one another, we obtain $\Exp{\inp{z^i}{z^j}} =0$ for $i \neq j$, and $\Exp{\sqnorm{z^k}}  = \frac{1}{M}\sum_{i=1}^M \Exp{\sqnorm{z_i^k}} \leq \sigma^2/M$. Therefore, 
\begin{eqnarray*}
	\Exp{\norm{e^{k+1}}} 
	\leq \Exp{\norm{e^0}} + \sqrt{ \beta^2 \frac{(K+1)\sigma^2}{M} }. 
\end{eqnarray*}	
\end{proof}

Next, by utilizing \Cref{lemma:bound_noise},  we obtain~\Cref{lemma:fedGD_dp_pp}, which bounds $\norm{ \frac{1}{M}\sum_{i=1}^M v_i^k - \hat v^k}$, the quantity that will be applied to conclude the convergence of \ouralg.

\begin{lemma}\label{lemma:fedGD_dp_pp}
Consider \ouralg~with any local updating operator $\mathcal{T}_i(\cdot)$ for solving Problem~\eqref{eqn:problem}, where~\Cref{assum:smooth} holds. Then,  
	\begin{eqnarray*}
	    \Exp{\norm{ \frac{1}{M}\sum_{i=1}^M v_i^{k+1} - \hat v^{k+1} }}	\leq \sqrt{ \frac{\beta^2 B}{M}  (K+1)  },
	\end{eqnarray*}
	where $B= 2p (1-1/p)^2 + 2(1-p) + 2\sigma^2_{\rm DP}/p$.
\end{lemma}

\begin{proof}
Define $e^k \eqdef \frac{1}{M}\sum_{i=1}^M v_i^{k} - \hat v^{k}$. Then, 
\begin{eqnarray*}
e^{k+1} 
& = & \frac{1}{M}\sum_{i=1}^M v_i^{k+1} - \hat v^{k+1} \\
& \overset{v_i^{k+1},\hat v^{k+1}}{=}& \frac{1}{M}\sum_{i=1}^M v_i^k - \hat v^k  + \beta n^k \\
& = & e^k + \beta n^k,
\end{eqnarray*}
where $n^k = \frac{1}{M}\sum_{i=1}^M n_i^k$ and  $n_i^k = (1-q_i^k) \Normalize{\alpha}{\frac{x^k - \mathcal{T}_i(x^k)}{\gamma}-v_i^k}  - q_i^k z_i^k$.

Next, from~\Cref{lemma:bound_noise} with $z^k=n^k$ and $z_i^k = n_i^k$, we obtain
\begin{eqnarray*}
	\Exp{\norm{e^{k+1}}} \leq \Exp{\norm{e^0}} + \sqrt{ \frac{\beta^2(K+1) \cdot \sigma^2 }{M} },
\end{eqnarray*}	
where $\sigma^2 = \Exp{\sqnorm{n_i^k}}$.

If $\hat v^0 = \frac{1}{n}\sum_{i=1}^n v_i^0$, then 
\begin{eqnarray*}
	\Exp{\norm{e^{k+1}}} \leq  \sqrt{ \frac{\beta^2(K+1) \cdot \sigma^2 }{M} }.
\end{eqnarray*}	

To complete the upper-bound for $\Exp{\norm{e^{k+1}}}$, 
we find $\sigma^2 = \Exp{\sqnorm{n_i^k}}$. From the definition of $n_i^k$,
\begin{eqnarray*}
	\Exp{\sqnorm{n_i^k}} & = & \Exp{  \sqnorm{ (1-q_i^k) \Normalize{\alpha}{\frac{x^k - \mathcal{T}_i(x^k)}{\gamma}-v_i^k}  - q_i^k z_i^k } } \\
	& \overset{(A)}{\leq} & 2 \Exp{ (1-q_i^k)^2 \sqnorm{\Normalize{\alpha}{\frac{x^k - \mathcal{T}_i(x^k)}{\gamma}-v_i^k}} } + 2\Exp{ (q_i^k)^2 \sqnorm{z_i^k}} \\
	& \overset{(B)}{\leq} & 2 \Exp{ (1-q_i^k)^2} + 2\Exp{(q_i^k)^2  \sqnorm{z_i^k}}\\
       & \overset{(C)}{=} &  2 \Exp{ (1-q_i^k)^2} + 2\Exp{(q_i^k)^2} \Exp{ \sqnorm{z_i^k} } \\
	& \leq & 2p (1-1/p)^2 + 2(1-p) + 2 p/p^2\cdot \sigma^2_{\rm DP},
\end{eqnarray*}	
where we reach $(A)$ by the fact that $\sqnorm{x+y} \leq 2 \sqnorm{x} + 2\sqnorm{y}$ for $x,y\in\R^d$, $(B)$ by the fact that $\norm{\Normalize{\alpha}{\cdot}}\leq 1$, and $(C)$ by the fact that $q_i^k$ and $z_i^k$ are independent. 

Finally, by plugging the upper-bound for $\Exp{\sqnorm{n_i^k}}$ into the upper-bound for $\Exp{\norm{e^{k+1}}}$, we obtain the final result.

\end{proof}

Next, \Cref{lemma:descentIneq} provides the descent inequality for $f(x^k)-f^{\inf}$ in normalized gradient descent.
\begin{lemma}\label{lemma:descentIneq}
Let $f:\R^d\rightarrow\R$ be lower-bounded by $f^{\inf} > -\infty$ and $L$-smoooth, and let $x^{k} \in \R^d$ be governed by 
\begin{eqnarray*}
	x^{k+1} = x^k - \gamma \frac{G^k}{\norm{G^k}},
\end{eqnarray*}	 	
where $\gamma>0$. Then, 
\begin{eqnarray*}
	f(x^{k+1}) - f^{\inf} \leq f(x^k) - f^{\inf} - \gamma \norm{\nabla f(x^k)} + 2\gamma \norm{\nabla f(x^k) - G^k} + \frac{L\gamma^2}{2}.
\end{eqnarray*}		
\end{lemma}
\begin{proof}
By the lower-bound and smoothess of $f(\cdot)$, and by the definition of $x^{k+1}$, 
\begin{eqnarray*}
f(x^{k+1}) - f^{\inf} 
& \leq& f(x^k) - f^{\inf} - \frac{\gamma}{\norm{G^k}} \inp{\nabla f(x^k)}{G^k} +  \frac{L\gamma^2}{2} \\
& \leq& f(x^k) - f^{\inf} - \frac{\gamma}{\norm{G^k}} \inp{G^k}{G^k} + \frac{\gamma}{\norm{G^k}} \inp{G^k-\nabla f(x^k)}{G^k} +  \frac{L\gamma^2}{2} \\
& = & f(x^k) - f^{\inf} - \gamma \norm{G^k} + \frac{\gamma}{\norm{G^k}} \inp{G^k-\nabla f(x^k)}{G^k} +  \frac{L\gamma^2}{2}.
\end{eqnarray*}		

By Cauchy-Scwartz inequality, i.e. $\inp{x}{y} \leq \norm{x}\norm{y}$ for $x,y\in\R^d$, 
\begin{eqnarray*}
	f(x^{k+1}) - f^{\inf} 
	& \leq& f(x^k) - f^{\inf} - \gamma \norm{G^k} + \gamma \norm{\nabla f(x^k)-G^k} +  \frac{L\gamma^2}{2}.
\end{eqnarray*}		

Finally, by the triangle inequality, 
\begin{eqnarray*}
	f(x^{k+1}) - f^{\inf} 
	& \leq& f(x^k) - f^{\inf} - \gamma \norm{\nabla f(x^k)} + 2\gamma \norm{\nabla f(x^k)-G^k} +  \frac{L\gamma^2}{2}.
\end{eqnarray*}		
\end{proof}

Finally, \Cref{lemma:localGD_convergence} derives the sublinear convergence from the given descent inequality. 
\begin{lemma}\label{lemma:localGD_convergence}
Let $\{V^k\}$, $\{W^k\}$ be non-negative sequences satisfying
\begin{eqnarray*}
    V^{k+1} \leq (1+b_1 \gamma^2) V^k - b_2\gamma W^k + b_3 \gamma.
\end{eqnarray*}
Then,
\begin{eqnarray*}
    \underset{k \in [0,K]}{\min} W^k \leq \frac{\exp(b_1\gamma^2(K+1))}{K+1} \frac{V^0}{b_2\gamma} + \frac{b_3}{b_2}.
\end{eqnarray*}
\end{lemma}
\begin{proof}
Define $w^k:=\frac{w^{k+1}}{1+b_1\gamma^2}$ for all $k \geq 0$. Then, 
\begin{eqnarray*}
   w^k W^k 
   &  \leq & \frac{w^k(1+b_1\gamma^2) V^k}{b_2\gamma} - \frac{w^kV^{k+1}}{b_2\gamma}  + \frac{b_3}{b_2}\\
   & = & \frac{w^{k-1}V^k - w^k V^{k+1}}{b_2\gamma} + \frac{b_3}{b_2}.
\end{eqnarray*}
By summing the inequality over $k=0,1,\ldots,K$, 
\begin{eqnarray*}
   \sum_{k=0}^{K} w^k W^k  
   & \leq & \frac{w^{-1}V^0 - w^K V^{K+1}}{b_2\gamma} + \frac{b_3}{b_2} \sum_{k=0}^K w^k \\
   & \overset{w^k,V^k \geq 0}{\leq} & \frac{w^{-1}V^0}{b_2\gamma} + \frac{b_3}{b_2} \sum_{k=0}^K w^k.
\end{eqnarray*}
Therefore,
\begin{eqnarray*}
    \min_{k\in[0,K]} W^k 
    & \leq & \frac{1}{\sum_{k=0}^K w^k}  \sum_{k=0}^{K} w^k W^k  \\
    & \leq & \frac{w^{-1}V^0}{b_2\gamma \sum_{k=0}^K w^k} + \frac{b_3}{b_2}. 
\end{eqnarray*}
Next, since 
\begin{eqnarray*}
    \sum_{k=0}^K w^k 
    & \geq & (K+1) \min_{k\in[0,K]} w^k  \\
    & = & (K+1) w^{K+1} \\
    & = & \frac{(K+1)w^{-1}}{(1+b_1\gamma^2)^{K+1}},
\end{eqnarray*}
we get 
\begin{eqnarray*}
    \min_{k\in[0,K]} W^k 
    \leq \frac{(1+ b_1\gamma^2)^{K+1}V^0}{b_2\gamma (K+1) } + \frac{b_3}{b_2}. 
\end{eqnarray*}

Finally, since $1+x \leq \exp(x)$, we have $(1+b_1\gamma^2)^{K+1} \leq \exp(b_1\gamma^2(K+1) )$. Using this result, we complete the proof.

\end{proof}

\newpage 
\section{Multiple Local GD Steps}

We derive the convergence theorem (\Cref{thm:DP_PP_full}) and associated corollaries of \ouralg using multiple local gradient descent (GD) steps.
The multiple local GD steps have the following update rule: 
$$\mathcal{T}_i(x^k)=x^k- \frac{\gamma}{T} \sum_{j=0}^{T-1} \nabla f_i(x_i^{k,j}).$$ 
Here, 
$x_i^{k,j}$ is updated according to: 
$$
    x^{k,j+1}_i = x^{k,j}_i - \frac{\gamma}{T} \nabla f_{i}(x^{k,j}_i) \quad \text{for} \quad j=0,1,\ldots,T-1.
$$

\subsection{Key Lemmas}

We begin by introducing two key lemmas for analyzing \ouralg using multiple local GD steps.
\Cref{lemma:localGD_avg_gradient_norm} bounds $\frac{1}{M}\sum_{i=1}^M \norm{\nabla f_i(x)}$, while~\Cref{lemma:localGD_recursion} proves the properties of local GD steps.

\begin{lemma}\label{lemma:localGD_avg_gradient_norm}
Let $f$ be bounded from below by $f^{\inf} > -\infty$, and each $f_i$ be bounded from below by $f_i^{\inf} > -\infty$ and $L$-smooth. Then, 
\begin{eqnarray*}
    \frac{1}{M}\sum_{i=1}^M \norm{\nabla f_i(x)} 
    \leq  \sqrt{\frac{2L}{\Delta^{\inf}}}[f(x)-f^{\inf}] + \sqrt{2L \Delta^{\inf}},
\end{eqnarray*}
where $\Delta^{\inf} = f^{\inf} - \frac{1}{M}\sum_{i=1}^M  f_i^{\inf}  \geq 0$.
\end{lemma}
\begin{proof}
Let $f$ be bounded from below by $f^{\inf} > -\infty$, and each $f_i$ be bounded from below by $f_i^{\inf} > -\infty$ and $L$-smooth. Then, 
\begin{eqnarray*}
    \sqnorm{\nabla f_i(x)} \leq 2L[f_i(x)-f_i^{\inf}].
\end{eqnarray*}
Therefore, 
\begin{eqnarray*}
    \frac{1}{M}\sum_{i=1}^M \sqnorm{\nabla f_i(x)} \leq A[f(x)-f^{\inf}] + B,
\end{eqnarray*}
where $A = 2L$, $B = 2L \Delta^{\inf}$, and $\Delta^{\inf}=f^{\inf} - \frac{1}{M}\sum_{i=1}^M  f_i^{\inf} \geq 0$. Thus, we obtain
\begin{eqnarray*}
    \frac{1}{M}\sum_{i=1}^M \norm{\nabla f_i(x)} 
    & \overset{\text{Jensen's inequality}}{\leq} & \sqrt{  \frac{1}{M}\sum_{i=1}^M \sqnorm{\nabla f_i(x)} } \\
    & \leq & \sqrt{A[f(x)-f^{\inf}] + B} \\
    & = & \frac{A[f(x)-f^{\inf}] + B }{\sqrt{A[f(x)-f^{\inf}] + B } } \\
    & \overset{f(x) \geq f^{\inf}}{\leq} & \frac{A}{\sqrt{B}}[f(x)-f^{\inf}] + \sqrt{B}.
\end{eqnarray*}
\end{proof}

\begin{lemma}\label{lemma:localGD_recursion}
Let each $f_i$ be $L$-smooth, and  let $\mathcal{T}_i(x^k)=x^k- \frac{\gamma}{T} \sum_{j=0}^{T-1} \nabla f_i(x_i^{k,j})$, where the sequence $\{x_i^{k,l}\}$ is generated by 
\begin{eqnarray*}
    x_i^{k,l+1} = x_i^{k,l} - \frac{\gamma}{T}\nabla f_i(x_i^{k,l}), \quad \text{for} \quad l=0,1,\ldots,T-1,
\end{eqnarray*}
given that $x_i^{k,0}=x^k$. If $\gamma \leq \frac{1}{2L}$, and $\norm{x^{k+1}-x^k} \leq \eta$ with $\eta>0$, then 
\begin{enumerate}
    \item $x_i^{k,l} = x^k - \frac{\gamma}{T}\sum_{j=0}^{l-1} \nabla f_i(x_i^{k,l})$.  
    \item $\frac{1}{T}\sum_{j=0}^{T-1} \norm{x_i^{k+1,j}-x_i^{k,j}} \leq 2\eta $. 
    \item $\frac{1}{T}\sum_{j=0}^{T-1}\norm{x^k-x_i^{k,j}} \leq 2\gamma  \norm{\nabla f_i(x^k)}$.
    \item $\norm{\mathcal{T}_i(x^{k+1}) - \mathcal{T}_i(x^k)} \leq 2\eta$.
    \item $\norm{(x^k-\gamma\nabla f_i(x^k)) - \mathcal{T}_i(x^k)} \leq 2L\gamma^2 \norm{\nabla f_i(x^k)}$.
\end{enumerate}
\end{lemma}
\begin{proof}
We prove the first statement by recursively applying the equation for $x_i^{k,j+1}$ for $j=0,1,\ldots,l-1$. 

Next, we prove the second statement. From the definition of the Euclidean norm, 
\begin{eqnarray*}
\norm{x^{k+1,l}_i - x_i^{k,l}}
& \overset{x_i^{k,j}}{=} & \norm{ x^{k+1} - x^k - \frac{\gamma}{T} \sum_{j=0}^{l-1}(\nabla f_i(x_i^{k+1,j}) - \nabla f_i(x_i^{k,j})) } \\
& \overset{\text{triangle inequality}}{\leq} & \norm{x^{k+1}-x^k} + \frac{\gamma}{T}\sum_{j=0}^{l-1}\norm{\nabla f_i(x_i^{k+1,j}) - \nabla f_i(x_i^{k,j})} \\
& \overset{\text{$L$-smoothness of $f_i$}}{\leq} & \norm{x^{k+1}-x^k} + \frac{L\gamma}{T}\sum_{j=0}^{l-1}\norm{x_i^{k+1,j} - x_i^{k,j}}.
\end{eqnarray*}

If $\norm{x^{k+1}-x^k} \leq \eta$ with $\eta>0$, then 
\begin{eqnarray*}
\norm{x^{k+1,l}_i - x_i^{k,l}}
 & \leq & \eta+ \frac{L\gamma}{T} \sum_{j=0}^{l-1} \norm{ x_i^{k+1,j} - x_i^{k,j} } \\
 & \overset{l \leq T}{\leq} & \eta+ \frac{L\gamma}{T} \sum_{j=0}^{T-1} \norm{ x_i^{k+1,j} - x_i^{k,j} }
\end{eqnarray*}
Therefore, 
\begin{eqnarray*}
\sum_{j=0}^{T-1}\norm{x^{k+1,j}_i - x_i^{k,j}} \leq \eta T+ L\gamma \sum_{j=0}^{T-1} \norm{ x_i^{k+1,j} - x_i^{k,j} }.
\end{eqnarray*}

If $\gamma \leq \frac{1}{2L}$, then $L\gamma \leq 1/2$, and 
\begin{eqnarray*}
\frac{1}{T}\sum_{j=0}^{T-1}\norm{x^{k+1,j}_i - x_i^{k,j}} \leq 2\eta .
\end{eqnarray*}

Next, we prove the third statement. From the definition of the Euclidean norm, 
\begin{eqnarray*}
    \norm{x^k - x_i^{k,j}} 
    & \overset{x_i^{k,l}}{=} & \norm{\frac{\gamma}{T}\sum_{j=0}^{l-1} \nabla f_i(x_i^{k,j})} \\
    & = & \norm{\frac{\gamma}{T}\sum_{j=0}^{l-1} [\nabla f_i(x_i^{k,j})-\nabla f_i(x^k) + \nabla f_i(x^k)]}. 
\end{eqnarray*}
By the triangle inequality, and by the $L$-smoothness of $f_i(\cdot)$, 
\begin{eqnarray*}
    \norm{x^k - x_i^{k,j}} 
    & \leq & \frac{\gamma}{T}\sum_{j=0}^{l-1} \norm{ \nabla f_i(x_i^{k,j})-\nabla f_i(x^k)} + \frac{\gamma}{T}\sum_{j=0}^{l-1}\norm{\nabla f_i(x^k)} \\
    & \leq & \frac{L\gamma}{T}\sum_{j=0}^{l-1} \norm{ x_i^{k,j}-x^k} + \frac{\gamma}{T}\sum_{j=0}^{l-1}\norm{\nabla f_i(x^k)}.
\end{eqnarray*}
By the fact that $l\leq T$ and that $\norm{x} \geq 0$ for $x\in\R^d$, 
\begin{eqnarray*}
    \norm{x^k - x_i^{k,j}} 
    & \leq & \frac{L\gamma}{T}\sum_{j=0}^{T-1} \norm{ x_i^{k,j}-x^k} + \gamma \norm{\nabla f_i(x^k)}.
\end{eqnarray*}
Therefore, 
\begin{eqnarray*}
    \sum_{j=0}^{T-1} \norm{x^k - x_i^{k,j}} 
    & \leq & L\gamma \sum_{j=0}^{T-1} \norm{ x_i^{k,j}-x^k} + \gamma T \norm{\nabla f_i(x^k)}.
\end{eqnarray*}

If $\gamma \leq \frac{1}{2L}$, then $L\gamma \leq 1/2$, and 
\begin{eqnarray*}
\sum_{j=0}^{T-1} \norm{x^k - x_i^{k,j}} \leq 2\gamma T \norm{\nabla f_i(x^k)}.
\end{eqnarray*}

Next, we prove the fourth statement. From the definition of $\mathcal{T}_i(x^k)$, 
\begin{eqnarray*}
    \norm{ \mathcal{T}_i(x^{k+1}) - \mathcal{T}_i(x^k)}
    & = & \norm{x^{k+1} - x^k - \frac{\gamma}{T}\sum_{j=0}^{l-1} [\nabla f_i(x_i^{k,l+1}) - \nabla f_i(x_i^{k,l})] }. 
\end{eqnarray*}
By the triangle inequality, and by the $L$-smoothness of $f_i(\cdot)$, 
\begin{eqnarray*}
    \norm{\mathcal{T}_i(x^{k+1}) - \mathcal{T}_i(x^k)}
    & \leq & \norm{x^{k+1} - x^k} + \frac{\gamma}{T}\sum_{j=0}^{l-1} \norm{\nabla f_i(x_i^{k,l+1}) - \nabla f_i(x_i^{k,l}) } \\
    & \leq & \norm{x^{k+1} - x^k} + \frac{L\gamma}{T}\sum_{j=0}^{l-1} \norm{ x_i^{k,l+1} - x_i^{k,l} }. 
\end{eqnarray*}
By the fact that $\norm{x^{k+1}-x^k} \leq \eta$, that $l \leq T$, and that $\sum_{j=0}^{T-1} \norm{x_i^{k+1,j}-x_i^{k,j}} \leq 2\eta T$, 
\begin{eqnarray*}
    \norm{\mathcal{T}_i(x^{k+1}) - \mathcal{T}_i(x^k)} \leq \eta + L\gamma \cdot 2\eta \overset{L\gamma \leq 1/2}{\leq} 2\eta.
\end{eqnarray*}

Finally, we prove the fifth statement. From the definition of $\mathcal{T}_i(x^k)$,
\begin{eqnarray*}
    \norm{(x^k-\gamma\nabla f_i(x^k)) - \mathcal{T}_i(x^k)}= \norm{ \left(x^k-\frac{\gamma}{T}\sum_{l=0}^{T-1}\nabla f_i(x^k) \right) - \left(x^k - \frac{\gamma}{T}\sum_{l=0}^{T-1} \nabla f_i(x_i^{k,l})\right) }.
\end{eqnarray*}
By the triangle inequality, the $L$-smoothness of $f_i(\cdot)$, and the fact that $$\sum_{j=0}^{T}\norm{x^k-x_i^{k,j}} \leq 2\gamma T \norm{\nabla f_i(x^k)},$$
we obtain
\begin{eqnarray*}
    \norm{(x^k-\gamma\nabla f_i(x^k)) - \mathcal{T}_i(x^k)}
    & \leq & \frac{\gamma}{T} \sum_{l=0}^{T-1} \norm{\nabla f_i(x^k)-\nabla f_i(x_i^{k,l})} \\
    & \leq & \frac{L\gamma}{T} \sum_{l=0}^{T-1} \norm{x^k -x_i^{k,l}} \\
    & \leq & 2L\gamma^2 \norm{\nabla f_i(x^k)}.
\end{eqnarray*}

\end{proof}

\subsection{Proof of Theorem~\ref{thm:DP_PP_full}}

We prove \Cref{thm:DP_PP_full}, which we restate below:

\begin{theorem*}[\ouralg with local GD steps]\label{thm:FedNormEC_GD_PP_DP_multiple}
Consider \ouralg~for solving Problem~\eqref{eqn:problem} where~\Cref{assum:smooth} holds.    
Let  $\mathcal{T}_i(x^k)=x^k- \gamma\frac{1}{T} \sum\limits_{j=0}^{T-1} \nabla f_{i}(x_i^{k,j})$, where the sequence $\{x_i^{k,j}\}$ is generated by 
$
    x_i^{k,j+1} = x_i^{k,j} - \frac{\gamma}{T}\nabla f_i(x_i^{k,j}), \quad \text{for} \quad j=0,1,\ldots,T-1,
$
given that $x_i^{k,0}=x^k$. Furthermore, let $\beta,\alpha>0$ be chosen such that $\frac{\beta}{\alpha+R} < 1$ with $R = \max_{i \in [1,M]} \norm{v_i^0 - \frac{x^0 - \mathcal{T}_i(x^0)}{\gamma}}$.
{\color{blue} If $\gamma = \frac{1}{2L}$ and $\eta \leq \min\left( \frac{1}{K+1}  \frac{\Delta^{\inf}}{2\sqrt{2L}} , \frac{1}{4L}\frac{\beta R}{\alpha + R} \right)$, then}
\begin{align*}
    \underset{k \in [0,K]}{\min} \Exp{\norm{\nabla f(x^k)}}  \leq& \frac{3}{K+1} \frac{f(x^0)-f^{\inf}}{\eta} + 2R + 2\sqrt{\frac{\beta^2 B}{M}(K+1)} \\ 
    & + \gamma \cdot   \mathbb{I}_{T\neq 1} \left[8L\sqrt{2L}  \sqrt{\Delta^{\inf}} \right]   + \eta \cdot \frac{L}{2},
\end{align*}
where 
$B= 2p (1-1/p)^2 + 2(1-p) + 2\sigma^2_{\rm DP}/p$, and $\Delta^{\inf} = f^{\inf} - \frac{1}{M}\sum_{i=1}^M f_i^{\inf} \geq 0$. 
\end{theorem*}

\begin{proof}

We prove the theorem in the following steps.

\paragraph{Step 1) Bound  $\norm{v_i^k - \frac{x^k - \mathcal{T}_i(x^k)}{\gamma}}$ by induction, and bound $\norm{v_i^{k+1} - \frac{x^k - \mathcal{T}_i(x^k)}{\gamma}}$.}
We prove by induction that  $$\norm{v_i^k - \frac{x^k - \mathcal{T}_i(x^k)}{\gamma}} \leq \max_{i\in[1,M]}\norm{v_i^0 - \frac{x^0 - \mathcal{T}_i(x^0)}{\gamma}}.$$ 
It is trivial to show the condition when $k=0$. Next, suppose that   $\norm{v_i^k - \frac{x^k - \mathcal{T}_i(x^k)}{\gamma}} \leq \max_{i\in[1,M]} \norm{v_i^0 - \frac{x^0 - \mathcal{T}_i(x^0)}{\gamma}}$ holds.  From~\Cref{lemma:localGD_recursion}, $\mathcal{T}_i(x^k)=x^k- \frac{\gamma}{T} \sum \limits_{j=0}^{T-1} \nabla f_i(x_i^{k,j})$ satisfies 
\begin{eqnarray*}
	\norm{\mathcal{T}_i(x^{k+1})-\mathcal{T}_i(x^k)} 
	\leq 2\eta.
\end{eqnarray*}
 Therefore, from Lemma~\ref{lemma:bound_v_and_T} with $\rho = 2$, $C = R = \max_{i\in[1,M]}\norm{v_i^0 -\frac{x^0 - \mathcal{T}_i(x^0)}{\gamma}}$, we can prove that by choosing $\frac{\beta}{\alpha + R} < 1$ and $\eta \leq \frac{\gamma\beta R}{(1+\rho) (\alpha +R)}$, $\norm{v_i^{k+1} - \frac{x^{k+1} - \mathcal{T}_i(x^{k+1})}{\gamma}} \leq R$. We complete the induction proof. 

Next, from Lemma~\ref{lemma:bound_v_and_T}, $\norm{v_i^{k+1} - \frac{x^k - \mathcal{T}_i(x^k)}{\gamma}} \leq \max_{i\in[1,M]}\norm{v_i^0 - \frac{x^0 - \mathcal{T}_i(x^0)}{\gamma}}$.

\paragraph{Step 2) Bound $f(x^k)-f^{\inf}$.}
From~\Cref{lemma:descentIneq} with $G^k = \hat v^{k+1}$, 
\begin{eqnarray*}
	f(x^{k+1}) - f^{\inf} 
	& \leq & 	f(x^{k}) - f^{\inf} - \eta \norm{\nabla f(x^k)} + 2\eta \norm{\nabla f(x^k)-\hat v^{k+1}} + \frac{L\eta^2}{2} \\
	& \overset{\text{triangle inequality}}{\leq} & 	f(x^{k}) - f^{\inf} - \eta \norm{\nabla f(x^k)} + 2\eta \norm{\nabla f(x^k)- v^{k+1}}  \\
    && + 2\eta \norm{\hat v^{k+1} - v^{k+1}} + \frac{L\eta^2}{2},
\end{eqnarray*}	
where  $v^{k+1}  =  \frac{1}{M}\sum^M_{i=1} v^{k+1}_i$.
Next, since 
\begin{eqnarray*}
	\norm{\nabla f(x^k)-v^{k+1}} 
	& = & \norm{ \nabla f(x^k) -     \frac{1}{M}\sum_{i=1}^M v^{k+1}_i  } \\
	& \overset{ \text{triangle inequality} }{\leq} & \frac{1}{M}\sum_{i=1}^M  \norm{v_i^{k+1}-  \nabla f_i(x^k)} \\
    & \overset{ \text{triangle inequality} }{\leq} & \frac{1}{M}\sum_{i=1}^M  \norm{v_i^{k+1}- \frac{x^k - \mathcal{T}_i(x^k)}{\gamma}}+ \frac{1}{M}\sum_{i=1}^M \norm{\frac{x^k - \mathcal{T}_i(x^k)}{\gamma} - \nabla f_i(x^k)},  
\end{eqnarray*}	
where $\mathcal{T}_i(x^k)=x^k- \frac{\gamma}{T} \sum_{j=0}^{T-1} \nabla f_i(x_i^{k,j})$, we get 
\begin{align*}
    	\norm{\nabla f(x^k)-v^{k+1}}  \leq \frac{1}{M}\sum_{i=1}^M  \norm{v_i^{k+1}- \frac{x^k - \mathcal{T}_i(x^k)}{\gamma}} + \frac{1}{\gamma}\frac{1}{M}\sum_{i=1}^M \norm{x^k - \mathcal{T}_i(x^k) - \gamma\nabla f_i(x^k)}.
\end{align*}
Plugging the upperbound for $\norm{\nabla f(x^k)-v^{k+1}}$ into the main inequality in $f(x^k)-f^{\inf}$, we obtain  
\begin{eqnarray*}
	f(x^{k+1}) - f^{\inf} 
	& \leq & 	f(x^{k}) - f^{\inf} - \eta \norm{\nabla f(x^k)} + 2\eta \frac{1}{M}\sum_{i=1}^M  \norm{v_i^{k+1} - \frac{x^k - \mathcal{T}_i(x^k)}{\gamma}}   \\
    && + \frac{2\eta}{\gamma} \frac{1}{M}\sum_{i=1}^M \norm{(x^k-\gamma\nabla f_i(x^k)) - \mathcal{T}_i(x^k)} + 2\eta \norm{\hat v^{k+1} - v^{k+1}}+ \frac{L\eta^2}{2}.
\end{eqnarray*}

By the fact that $\norm{v_i^{k+1} - \frac{x^k - \mathcal{T}_i(x^k)}{\gamma}} \leq R$ from Step 1), 
\begin{eqnarray*}
	f(x^{k+1}) - f^{\inf} 
	& \leq & 	f(x^{k}) - f^{\inf} - \eta \norm{\nabla f(x^k)} + 2\eta  R  \\
     && + \frac{2\eta}{\gamma} \frac{1}{M}\sum_{i=1}^M \norm{(x^k-\gamma\nabla f_i(x^k)) - \mathcal{T}_i(x^k)} + 2\eta \norm{\hat v^{k+1} - v^{k+1}} + \frac{L\eta^2}{2}.
\end{eqnarray*}

To complete the proof, we consider two possible cases for $\mathcal{T}_i(x^k)$: 1) when $T=1$ and 2) when $T \neq 1$.

\paragraph{Case 1) $\mathcal{T}_i(x^k)$ with $T=1$.}
When $\mathcal{T}_i(x^k)$ with $T=1$, $\norm{(x^k-\gamma\nabla f_i(x^k)) - \mathcal{T}_i(x^k)} = 0$, and 
\begin{eqnarray*}
	f(x^{k+1}) - f^{\inf} 
	& \leq & 	f(x^{k}) - f^{\inf} - \eta \norm{\nabla f(x^k)} + 2\eta  R   + 2\eta \norm{\hat v^{k+1} - v^{k+1}} + \frac{L\eta^2}{2}.
\end{eqnarray*}

\paragraph{Case 2) $\mathcal{T}_i(x^k)$ with $T>1$.}
When $\mathcal{T}_i(x^k)$ with $T>1$, from~\Cref{lemma:localGD_recursion},  
\begin{eqnarray*}
	f(x^{k+1}) - f^{\inf} 
	& \leq & 	f(x^{k}) - f^{\inf} - \eta \norm{\nabla f(x^k)} + 2\eta  R  \\
     && + 4L \gamma \eta \frac{1}{M}\sum_{i=1}^M \norm{\nabla f_i(x^k)} + 2\eta \norm{\hat v^{k+1} - v^{k+1}} + \frac{L\eta^2}{2}.
\end{eqnarray*}
Therefore, from two cases, we obtain the descent inequality,
\begin{eqnarray*}
	f(x^{k+1}) - f^{\inf} 
	& \leq & 	f(x^{k}) - f^{\inf} - \eta \norm{\nabla f(x^k)} + 2\eta  R  \\
     && + \mathbb{I}_{T\neq 1} \left[ 4L \gamma \eta \frac{1}{M}\sum_{i=1}^M \norm{\nabla f_i(x^k)} \right] + 2\eta \norm{\hat v^{k+1} - v^{k+1}} + \frac{L\eta^2}{2}.
\end{eqnarray*}
Next, from~\Cref{lemma:localGD_avg_gradient_norm}, 
\begin{eqnarray*}
	f(x^{k+1}) - f^{\inf} 
	& \leq & 	\left( 1+ \frac{4L\sqrt{2L}}{\sqrt{\Delta^{\inf}}}\gamma\eta \right)(f(x^{k}) - f^{\inf}) - \eta \norm{\nabla f(x^k)} + 2\eta  R  \\
     && + \mathbb{I}_{T\neq 1} \left[ 4L \sqrt{2L} \gamma \eta  \sqrt{\Delta^{\inf}} \right] + 2\eta \norm{\hat v^{k+1} - v^{k+1}} + \frac{L\eta^2}{2}.
\end{eqnarray*}
Since
\begin{eqnarray*}
    \Exp{\norm{\hat v^{k+1} - v^{k+1}}} 
    & \leq & \frac{1}{\gamma} \Exp{\norm{ \frac{1}{M}\sum_{i=1}^M v_i^{k+1} - \hat v^{k+1} }}  \\
    & \overset{\text{\Cref{lemma:fedGD_dp_pp}}}{\leq} & \frac{1}{\gamma}\sqrt{\frac{\beta^2 B}{M}(K+1)},
\end{eqnarray*}
by taking the expectation, 
\begin{eqnarray*}
	\Exp{f(x^{k+1}) - f^{\inf}} 
	& \leq & 	\left( 1+ \frac{4L\sqrt{2L}}{\sqrt{\Delta^{\inf}}}\gamma\eta \right)\Exp{f(x^{k}) - f^{\inf}} - \eta \Exp{\norm{\nabla f(x^k)}} + 2\eta  R  \\
     && + \mathbb{I}_{T\neq 1} \left[ 8L \sqrt{2L} \gamma \eta  \sqrt{\Delta^{\inf}} \right] + 2\eta \sqrt{\frac{\beta^2 B}{M}(K+1)} + \frac{L\eta^2}{2}.
\end{eqnarray*}

By applying~\Cref{lemma:localGD_convergence} with $ \eta\gamma \leq \frac{1}{K+1} \frac{\Delta^{\inf}}{4L\sqrt{2L}} $ and using the fact $(1 + \eta\gamma \frac{4L\sqrt{2L}}{\Delta^{\inf}})^{K+1} \leq \exp(\eta\gamma\frac{4L\sqrt{2L}}{\Delta^{\inf}}(K+1)) \leq \exp(1) \leq 3$ we finalize the proof.

\end{proof}

\subsection{Corollaries for \ouralg with multiple local GD steps from~\Cref{thm:DP_PP_full}}

\subsubsection{\Cref{cor:conv_local_GD}}

\Cref{cor:conv_local_GD} presents the $\cO(1/K^{1/8})$ convergence of \ouralg with multiple local GD steps ($T>1$).

\begin{corollary}[Convergence bound for \ouralg with multiple local GD steps]
\label{cor:conv_local_GD}
Consider  \ouralg~ for solving Problem~\eqref{eqn:problem} under the same setting as~\Cref{thm:DP_PP_full}. 
 Let $T>1$ (multiple local GD steps).
 If $\gamma = \frac{1}{2L (K+1)^{1/8}}$, $v_i^0\in\R^d$ is chosen such that  $\max_{i\in[1,M]} \norm{\frac{x^0 - \mathcal{T}_i(x^0)}{\gamma} - v_i^0 } = \frac{D_1}{(K+1)^{1/8}}$ with $D_1 >0$, and $\beta = \frac{D_2}{(K+1)^{5/8}}$ with $D_2 >0$, and $\eta = \frac{\hat \eta}{(K+1)^{7/8}}$ with $\hat \eta = \min\left( \frac{\Delta^{\inf}}{2\sqrt{2L}}, \frac{D_1 D_2}{4L(\alpha + D_1)}   \right)$, then  
 \begin{eqnarray*}
	 \underset{k\in [0,K]}{\min} \Exp{\norm{\nabla f(x^k)}}
	\leq  \frac{A_1}{(K+1)^{1/8}} + \frac{A_2}{(K+1)^{7/8}},
\end{eqnarray*}	  
where $A_1 = 3 \frac{f(x^0)-f^{\inf}}{\hat \eta} + 2D_1  + \frac{2\sqrt{B} D_2}{\sqrt{M}} + 4\sqrt{2L}  \sqrt{\Delta^{\inf}}$ and $A_2 = \hat \eta L /2$.
\end{corollary}
\begin{proof}
Let $T>1$. Then, from~\Cref{thm:DP_PP_full}, 
\begin{align*}
    \underset{k \in [0,K]}{\min} \Exp{\norm{\nabla f(x^k)}}  \leq& \frac{3}{K+1} \frac{f(x^0)-f^{\inf}}{\eta} + 2R + 2\sqrt{\frac{\beta^2 B}{M}(K+1)} + \eta \cdot \frac{L}{2} \\ 
    & + \gamma \cdot \left[8L\sqrt{2L}  \sqrt{\Delta^{\inf}} \right],
\end{align*}
where $B = 2p(1-1/p)^2 + 2 (1-p)+ 2\sigma_{\rm DP}^2/p$.

Next, suppose that
\begin{itemize}
    \item $\gamma = \frac{1}{2L(K+1)^{1/8}}$ to guarantee that $\gamma \leq 1/(2L)$
    \item $v_i^0 \in \R^d$ such that $\max_{i\in[1,M]} \norm{\frac{x^0 - \mathcal{T}_i(x^0)}{\gamma} - v_i^0 } = R = \frac{D_1}{(K+1)^{1/8}}$ with $D_1 >0$ 
    \item $\beta = \frac{D_2}{(K+1)^{5/8}}$ with $D_2 >0$.
\end{itemize}
Then, we choose  $\eta = \frac{\hat \eta}{(K+1)^{7/8}}$ with $\hat \eta = \min\left( \frac{\Delta^{\inf}}{2\sqrt{2L}}, \frac{D_1 D_2}{4L(\alpha + D_1)}   \right)$ to ensure that $ \eta\gamma \leq \frac{1}{K+1} \frac{\Delta^{\inf}}{4L\sqrt{2L}} $ and  $\eta \leq \frac{\gamma}{2}\frac{\beta R}{\alpha + R}$. Therefore, 
\begin{align*}
    \underset{k \in [0,K]}{\min} \Exp{\norm{\nabla f(x^k)}}  \leq& \frac{A_1}{(K+1)^{1/8}} + \frac{A_2}{(K+1)^{7/8}},
\end{align*}
where $A_1 = 3 \frac{f(x^0)-f^{\inf}}{\hat \eta} + 2D_1  + \frac{2\sqrt{B} D_2}{\sqrt{M}} + 4\sqrt{2L}  \sqrt{\Delta^{\inf}}$ and $A_2 = \hat \eta L /2$.

\end{proof}

\subsubsection{\Cref{corr:DP_PP_utility_local_GD}}

\Cref{corr:DP_PP_utility_local_GD} establishes  the utility bound of \ouralg with multiple local GD steps ($T>1$). Here, we choose $\sigma_{\rm DP}= c \frac{p \sqrt{(K+1)\log(1/\delta)}}{\epsilon}$ with $c >0$, and let $p=\frac{\hat B}{M}$ for $\hat B\in[1,M]$ is the number of clients being sampled on each  round.


\begin{corollary}[Utility bound for \ouralg with multiple local GD steps]
\label{corr:DP_PP_utility_local_GD}
Consider \ouralg~ for solving Problem~\eqref{eqn:problem} under the same setting as~\Cref{thm:DP_PP_full}.
Let $T>1$ (multiple local GD steps), let $\sigma_{\rm DP}= c \frac{p \sqrt{(K+1)\log(1/\delta)}}{\epsilon}$ with $c >0$ (privacy with subsampling amplification), and let $p=\frac{\hat B}{M}$ where $\hat B\in[1,M]$ is the number of clients being sampled on each round.
If $\gamma < \frac{2(\Delta^{\inf})^2}{3L(f(x^0)-f^{\inf})\sqrt{M/B_2}}$,
$\beta = \frac{\hat \beta}{K+1}$ with $\hat \beta = \sqrt{\frac{3(f(x^0)-f^{\inf})}{\gamma}}\sqrt[4]{\frac{M}{B_2}}$,
$\alpha = R = \cO\left( \sqrt[4]{d}\frac{\sqrt{f(x^0)-f^{\inf}}}{\sqrt{\gamma}}\sqrt[4]{\frac{B_2}{M}} \right)$ with $B_2 = 2c^2 \frac{\hat B}{M} \frac{ \log(1/\delta)}{\epsilon^2}$, and
$\eta = \frac{1}{K+1} \frac{\gamma}{2}\frac{\hat \beta R}{\alpha + R}$, then
\begin{eqnarray*}
     \underset{k \in [0,K]}{\min} \Exp{\norm{\nabla f(x^k)}} = \cO\left(\Delta \sqrt[4]{\frac{d\hat B}{M^2}\frac{\log(1/\delta)}{ \epsilon^2}} + \sqrt{L}  \sqrt{\Delta^{\inf}}\right),
\end{eqnarray*}
where $\Delta = \max(\alpha,2)\sqrt{L} \sqrt{f(x^0)-f^{\inf}}$.
\end{corollary}

\begin{proof}
Let $T>1$. Then, from~\Cref{thm:DP_PP_full}, 
\begin{align*}
    \underset{k \in [0,K]}{\min} \Exp{\norm{\nabla f(x^k)}}  \leq& \frac{3}{K+1} \frac{f(x^0)-f^{\inf}}{\eta} + 2R + 2\sqrt{\frac{\beta^2 B}{M}(K+1)} + \eta \cdot \frac{L}{2} \\ 
    & + \gamma \cdot \left[8L\sqrt{2L}  \sqrt{\Delta^{\inf}} \right],
\end{align*}
where $B = 2p(1-1/p)^2 + 2 (1-p)+ 2\sigma_{\rm DP}^2/p$.

Also, let $\sigma_{\rm DP}= c \frac{p \sqrt{(K+1)\log(1/\delta)}}{\epsilon}$ with $c >0$, and let $p=\frac{\hat B}{M}$ where $\hat B\in[1,M]$ is the number of clients being sampled on each round. Then, $B = \frac{2\hat B}{M}\left( 1- \frac{M}{\hat B}\right)^2 + 2 \left(1 - \frac{\hat B}{M} \right) + 2 \frac{c\sqrt{K+1}\log(1/\delta)}{\epsilon}$, and 
\begin{align*}
    \underset{k \in [0,K]}{\min} \Exp{\norm{\nabla f(x^k)}}  & \leq \frac{3}{K+1} \frac{f(x^0)-f^{\inf}}{\eta} + 2R + 2\beta  \sqrt{\frac{B_1}{M}(K+1)} + 2\beta \sqrt{\frac{B_2}{M}}(K+1)  + \eta \cdot \frac{L}{2}\\
     & + \gamma \cdot \left[8L\sqrt{2L}  \sqrt{\Delta^{\inf}} \right],
\end{align*}
where $B_1 = \frac{2\hat B}{M}\left[ \left( 1- \frac{M}{\hat B}\right)^2 + \frac{M}{\hat B} -1 \right]$ and $B_2 = 2c^2 \frac{\hat B}{M} \frac{ \log(1/\delta)}{\epsilon^2}$.

If $\beta = \frac{\hat\beta}{K+1}$ with $\hat \beta >0$, then 
\begin{align*}
    \underset{k \in [0,K]}{\min} \Exp{\norm{\nabla f(x^k)}}  & \leq \frac{3}{K+1} \frac{f(x^0)-f^{\inf}}{\eta} + 2R + 2\hat \beta  \sqrt{\frac{B_1}{M(K+1)}} + 2\hat\beta \sqrt{\frac{B_2}{M}}  + \eta \cdot \frac{L}{2}\\
     & + \gamma \cdot \left[8L\sqrt{2L}  \sqrt{\Delta^{\inf}} \right].
\end{align*}
Since $\beta = \frac{\hat\beta}{K+1}$, we require
\begin{eqnarray*}
 \eta \leq \frac{1}{K+1} \min\left(  \frac{\Delta^{\inf}}{2\sqrt{2L}} ,   \frac{\gamma}{2}\frac{\hat \beta R}{\alpha + R} \right).
\end{eqnarray*}
Substituting $\hat \beta$ and setting $\alpha=R$, the condition $\frac{\Delta^{\inf}}{2\sqrt{2L}} > \frac{\gamma}{2}\frac{\hat \beta R}{\alpha + R}$ is equivalent to $\gamma < \frac{2(\Delta^{\inf})^2}{3L(f(x^0)-f^{\inf})\sqrt{M/B_2}}$. Since we choose $\gamma$ to satisfy this bound, we obtain
\begin{eqnarray*}
 \eta \leq \frac{1}{K+1}    \frac{\gamma}{2}\frac{\hat \beta R}{\alpha + R}.
\end{eqnarray*}

If $\eta = \frac{1}{K+1}    \frac{\gamma}{2}\frac{\hat \beta R}{\alpha + R}$, then 
\begin{align*}
    \underset{k \in [0,K]}{\min} \Exp{\norm{\nabla f(x^k)}}   \leq & 
    \frac{6\alpha(f(x^0)-f^{\inf})}{\gamma \hat\beta R} + 
    \frac{6(f(x^0)-f^{\inf})}{\gamma \hat \beta } + 2R + 2\hat\beta \sqrt{\frac{B_2}{M}}   \\
    &  + 2\hat \beta  \sqrt{\frac{B_1}{M(K+1)}} + \frac{1}{K+1} \cdot \frac{\gamma L \hat\beta R}{4(\alpha + R)}\\
     & + \gamma \cdot \left[8L\sqrt{2L}  \sqrt{\Delta^{\inf}} \right].
\end{align*}

If $\hat \beta = \sqrt{\frac{3(f(x^0)-f^{\inf})}{\gamma}}\sqrt[4]{\frac{M}{B_2}}$, then
\begin{align*}
    \underset{k \in [0,K]}{\min} \Exp{\norm{\nabla f(x^k)}}   \leq & \frac{2\sqrt{3}\alpha\sqrt{f(x^0)-f^{\inf}}}{\sqrt{\gamma} R}\sqrt[4]{\frac{B_2}{M}} + \frac{4\sqrt{3}\sqrt{f(x^0)-f^{\inf}}}{\sqrt{\gamma}}\sqrt[4]{\frac{B_2}{M}} + 2R   \\
    &  + 2\hat \beta  \sqrt{\frac{B_1}{M(K+1)}} + \frac{1}{K+1} \cdot \frac{\gamma L \hat\beta R}{4(\alpha + R)}   + \gamma \cdot \left[8L\sqrt{2L}  \sqrt{\Delta^{\inf}} \right].
\end{align*}

If $\alpha = R = \cO\left( \sqrt[4]{d}\frac{\sqrt{f(x^0)-f^{\inf}}}{\sqrt{\gamma}}\sqrt[4]{\frac{B_2}{M}} \right)$, then 
\begin{align*}
     \underset{k \in [0,K]}{\min} \Exp{\norm{\nabla f(x^k)}}   \leq & \cO\left(\Delta \sqrt[4]{\frac{d \hat B}{M^2}\frac{\log(1/\delta)}{\epsilon^2}} \right) + \cO\left(\frac{1}{\sqrt{K+1}}\right) + \cO\left( \frac{1}{K+1}\right) \\
     &  + \gamma \cdot \left[8L\sqrt{2L}  \sqrt{\Delta^{\inf}} \right] \\
     \leq & \cO\left(\Delta \sqrt[4]{\frac{d\hat B}{M^2}\frac{\log(1/\delta)}{\epsilon^2}} + \gamma \cdot \left[8L\sqrt{2L}  \sqrt{\Delta^{\inf}} \right]\right) \\
     & + \cO\left(\frac{1}{\sqrt{K+1}}\right) + \cO\left( \frac{1}{K+1}\right),
\end{align*}
where $\Delta = 2\sqrt{3}\max(\alpha,2)$. Finally, since $\gamma \le 1/(2L)$ and satisfies the derived upper bound, we complete the proof. 
    
\end{proof}

\subsubsection{\Cref{corr:nonprivate_PP}}

\Cref{corr:nonprivate_PP} states the $\cO(1/K^{1/6})$ convergence of \ouralg with one local GD step ($T=1$).

\begin{corollary}[Convergence bound for \ouralg with one local GD step]\label{corr:nonprivate_PP}
 Consider  \ouralg~ for solving Problem~\eqref{eqn:problem} under the same setting as~\Cref{thm:DP_PP_full}. 
 Let $T=1$  (one local GD step). 
 If $\gamma = \frac{1}{2L}$, $v_i^0\in\R^d$ is chosen such that $\max_{i\in[1,M]} \norm{\frac{x^0 - \mathcal{T}_i(x^0)}{\gamma} - v_i^0 } = \frac{D_1}{(K+1)^{1/6}}$ with $D_1 >0$, and $\beta = \frac{D_2}{(K+1)^{2/3}}$ with $D_2 >0$, and $\eta = \frac{\hat\eta}{(K+1)^{5/6}}$ with $\hat\eta = \frac{ D_1 D_2}{ 4L(\alpha + D_1 )}$, then  
 \begin{eqnarray*}
	 \underset{k\in [0,K]}{\min} \Exp{\norm{\nabla f(x^k)}}
	\leq  \frac{A_1}{(K+1)^{1/6}} + \frac{A_2}{(K+1)^{5/6}},
\end{eqnarray*}	
where $A_1 = 3 \frac{f(x^0)-f^{\inf}}{\hat \eta} + 2D_1  + \frac{2\sqrt{B} D_2}{\sqrt{M}}$ and $A_2 = \hat \eta L /2$.
\end{corollary}
\begin{proof}
Let $T=1$. Then, from~\Cref{thm:DP_PP_full}, 
\begin{align*}
    \underset{k \in [0,K]}{\min} \Exp{\norm{\nabla f(x^k)}}  \leq& \frac{3}{K+1} \frac{f(x^0)-f^{\inf}}{\eta} + 2R + 2\sqrt{\frac{\beta^2 B}{M}(K+1)} + \eta \cdot \frac{L}{2},
\end{align*}
where $B = 2p(1-1/p)^2 + 2 (1-p)+ 2\sigma_{\rm DP}^2/p$.

Next, suppose that
\begin{itemize}
    \item $\gamma = \frac{1}{2L}$
    \item $v_i^0 \in \R^d$ such that $\max_{i\in[1,M]} \norm{\frac{x^0 - \mathcal{T}_i(x^0)}{\gamma} - v_i^0 } = R = \frac{D_1}{(K+1)^{1/6}}$ with $D_1 >0$ 
    \item $\beta = \frac{D_2}{(K+1)^{2/3}}$ with $D_2 >0$.
\end{itemize}
Then, we choose  $\eta = \frac{\hat \eta}{(K+1)^{5/6}}$ with $\hat \eta =  \frac{D_1 D_2}{4L(\alpha + D_1)} $ to ensure that  $\eta \leq \frac{\gamma}{2}\frac{\beta R}{\alpha + R}$. Therefore, 
\begin{align*}
    \underset{k \in [0,K]}{\min} \Exp{\norm{\nabla f(x^k)}}  \leq& \frac{A_1}{(K+1)^{1/6}} + \frac{A_2}{(K+1)^{5/6}},
\end{align*}
where $A_1 = 3 \frac{f(x^0)-f^{\inf}}{\hat \eta} + 2D_1  + \frac{2\sqrt{B} D_2}{\sqrt{M}}$ and $A_2 = \hat \eta L /2$.

\end{proof}

\subsubsection{\Cref{corr:DP_PP_utility}}

\Cref{corr:DP_PP_utility} shows the utility bound of \ouralg with one local GD step ($T=1$). 
Here, we choose $\sigma_{\rm DP}= c \frac{p \sqrt{(K+1)\log(1/\delta)}}{\epsilon}$ with $c >0$, and let $p=\frac{\hat B}{M}$ for $\hat B\in[1,M]$ is the number of clients being sampled on each  round.

\begin{corollary*}[Utility bound for \ouralg with one local GD step]
 Consider  \ouralg~ for solving Problem~\eqref{eqn:problem} under the same setting as~\Cref{thm:DP_PP_full}. 
 Let $T=1$ (one local GD step),  let $\sigma_{\rm DP}= c \frac{p \sqrt{(K+1)\log(1/\delta)}}{\epsilon}$ with $c >0$ (privacy  with subsampling amplification), and let $p=\frac{\hat B}{M}$ for $\hat B\in[1,M]$ (client subsampling).
If $\beta = \frac{\hat \beta}{K+1}$ with $\hat \beta = \sqrt{\frac{3(f(x^0)-f^{\inf})}{\gamma}}\sqrt[4]{\frac{M}{B_2}}$, 
$\gamma < \frac{\Delta^{\inf} (\alpha + R)}{\sqrt{2L} \hat\beta R}$,
$\alpha = R = \cO\left( \sqrt[4]{d}\frac{\sqrt{f(x^0)-f^{\inf}}}{\sqrt{\gamma}}\sqrt[4]{\frac{B_2}{M}} \right)$ with $B_2 = 2c^2 \frac{\hat B}{M} \frac{ \log(1/\delta)}{\epsilon^2}$, and 
$\eta = \frac{1}{K+1}    \frac{\gamma}{2}\frac{\hat \beta R}{\alpha + R}$, then
\begin{eqnarray*}
     \underset{k \in [0,K]}{\min} \Exp{\norm{\nabla f(x^k)}}  = \cO\left(\Delta \sqrt[4]{\frac{d\hat B}{M^2}\frac{\log(1/\delta)}{ \epsilon^2}} \right), 
\end{eqnarray*}
where $\Delta = \max(\alpha,2)\sqrt{L} \sqrt{f(x^0)-f^{\inf}}$.
\end{corollary*}
\begin{proof}
Let $T=1$. Then, from~\Cref{thm:DP_PP_full}, 
\begin{align*}
    \underset{k \in [0,K]}{\min} \Exp{\norm{\nabla f(x^k)}}  \leq \frac{3}{K+1} \frac{f(x^0)-f^{\inf}}{\eta} + 2R + 2\sqrt{\frac{\beta^2 B}{M}(K+1)}   + \eta \cdot \frac{L}{2},
\end{align*}
where $B = 2p(1-1/p)^2 + 2 (1-p)+ 2\sigma_{\rm DP}^2/p$.

Also, let $\sigma_{\rm DP}= c \frac{p \sqrt{(K+1)\log(1/\delta)}}{\epsilon}$ with $c >0$, and let $p=\frac{\hat B}{M}$ for $\hat B\in[1,M]$ is the number of clients being sampled on each  round. Then, $B = \frac{2\hat B}{M}\left( 1- \frac{M}{\hat B}\right)^2 + 2 \left(1 - \frac{\hat B}{M} \right) + 2 \frac{c\sqrt{K+1}\log(1/\delta)}{\epsilon}$, and 
\begin{align*}
    \underset{k \in [0,K]}{\min} \Exp{\norm{\nabla f(x^k)}}  \leq \frac{3}{K+1} \frac{f(x^0)-f^{\inf}}{\eta} + 2R + 2\beta  \sqrt{\frac{B_1}{M}(K+1)} + 2\beta \sqrt{\frac{B_2}{M}}(K+1)  + \eta \cdot \frac{L}{2},
\end{align*}
where $B_1 = \frac{2\hat B}{M}\left[ \left( 1- \frac{M}{\hat B}\right)^2 + \frac{M}{\hat B} -1 \right]$ and $B_2 = 2c^2 \frac{\hat B}{M} \frac{ \log(1/\delta)}{\epsilon^2}$.

If $\beta = \frac{\hat\beta}{K+1}$ with $\hat \beta >0$, then 
\begin{align*}
    \underset{k \in [0,K]}{\min} \Exp{\norm{\nabla f(x^k)}}  \leq \frac{3}{K+1} \frac{f(x^0)-f^{\inf}}{\eta} + 2R + 2\hat \beta  \sqrt{\frac{B_1}{M(K+1)}} + 2\hat\beta \sqrt{\frac{B_2}{M}}  + \eta \cdot \frac{L}{2}.
\end{align*}
Since $\beta = \frac{\hat\beta}{K+1}$, we obtain
\begin{eqnarray*}
 \eta \leq \frac{1}{K+1} \min\left(  \frac{\Delta^{\inf}}{2\sqrt{2L}} ,   \frac{\gamma}{2}\frac{\hat \beta R}{\alpha + R} \right).
\end{eqnarray*}

If $\Delta^{\inf} > \frac{\gamma \sqrt{2L} \hat\beta R}{\alpha +R}$, then 
\begin{eqnarray*}
 \eta \leq \frac{1}{K+1}    \frac{\gamma}{2}\frac{\hat \beta R}{\alpha + R}.
\end{eqnarray*}

If $\eta = \frac{1}{K+1}    \frac{\gamma}{2}\frac{\hat \beta R}{\alpha + R}$, then 
\begin{align*}
    \underset{k \in [0,K]}{\min} \Exp{\norm{\nabla f(x^k)}}   \leq & 
    \frac{6\alpha(f(x^0)-f^{\inf})}{\gamma \hat\beta R} + 
    \frac{6(f(x^0)-f^{\inf})}{\gamma \hat \beta } + 2R + 2\hat\beta \sqrt{\frac{B_2}{M}}   \\
    &  + 2\hat \beta  \sqrt{\frac{B_1}{M(K+1)}} + \frac{1}{K+1} \cdot \frac{\gamma L \hat\beta R}{4(\alpha + R)}.
\end{align*}

If $\hat \beta = \sqrt{\frac{3(f(x^0)-f^{\inf})}{\gamma}}\sqrt[4]{\frac{M}{B_2}}$, then
\begin{align*}
    \underset{k \in [0,K]}{\min} \Exp{\norm{\nabla f(x^k)}}   \leq & \frac{2\sqrt{3}\alpha\sqrt{f(x^0)-f^{\inf}}}{\sqrt{\gamma} R}\sqrt[4]{\frac{B_2}{M}} + \frac{4\sqrt{3}\sqrt{f(x^0)-f^{\inf}}}{\sqrt{\gamma}}\sqrt[4]{\frac{B_2}{M}} + 2R   \\
    &  + 2\hat \beta  \sqrt{\frac{B_1}{M(K+1)}} + \frac{1}{K+1} \cdot \frac{\gamma L \hat\beta R}{4(\alpha + R)}.
\end{align*}

If $\alpha = R = \cO\left( \sqrt[4]{d}\frac{\sqrt{f(x^0)-f^{\inf}}}{\sqrt{\gamma}}\sqrt[4]{\frac{B_2}{M}} \right)$, then 
\begin{eqnarray*}
     \underset{k \in [0,K]}{\min} \Exp{\norm{\nabla f(x^k)}}  \leq \cO\left(\Delta \frac{\sqrt{f(x^0)-f^{\inf}}}{\sqrt{\gamma}}\sqrt[4]{d\frac{B_2}{M}} \right) + \cO\left(\frac{1}{\sqrt{K+1}}\right) + \cO\left( \frac{1}{K+1}\right), 
\end{eqnarray*}
where $\Delta = 2\sqrt{3}\max(\alpha,2)$. Finally, if $\gamma = 1/(2L)$, then we complete the proof. 
\end{proof}

\newpage 
\section{Multiple Local IG steps}

In this section, we derive the convergence theorem of \ouralg with multiple local steps using the Incremental Gradient (IG) method. 
The IG method has the following update rule. 
\begin{align}\label{eqn:fixed_point_localIG}
    \mathcal{T}^{IG}_i(x^k) = x^k - \gamma \frac{1}{N} \sum_{j=0}^{N-1} \nabla f_{i,j}(x^{k,j}_i),
\end{align} 
where $x_i^{k,j}$ is updated according to: 
$$
    x^{k,j+1}_i = x^{k,j}_i - \frac{\gamma}{N} \nabla f_{i,j}(x^{k,j}_i) \quad \text{for} \quad j=0,1,\ldots,T-1.
$$
In the update rule of the IG method, the number of local steps is equal to the size of the local data set. This implies that each client performs local updates $\mathcal{T}_i^{IG}(\cdot)$ using their entire local dataset. Furthermore, the IG method employs a fixed, deterministic permutation for its cyclic updates, unlike the well-known Random Reshuffling method.

\subsection{Key Lemmas}
First, we introduce key lemmas for analyzing \ouralg using multiple local IG steps.
\Cref{lemma:localIG_avg_gradient_norm} bounds     $\frac{1}{M}\sum^M_{i=1} \frac{1}{N}\sum^{N-1}_{j=0}\left\|x^{k,j}_i - x^k \right\|$ while~\Cref{lemma:localIG_recursion} proves the properties of local IG steps. 
\begin{lemma}
    \label{lemma:localIG_avg_gradient_norm}
    Consider the local IG method updates in~\eqref{eqn:fixed_point_localIG}. Let $f$ be bounded from below by $f^{\inf} > -\infty$, let each $f_i$ be bounded from below by $f_i^{\inf} > -\infty$, and let each $f_{i,j}$ be bounded from below by $f_{i,j}^{\inf}$ and $L$-smooth. Then, 
    \begin{align*}
    \frac{1}{M}\sum^M_{i=1} \frac{1}{N}\sum^{N-1}_{j=0}\left\|x^{k,j}_i - x^k \right\| 
        &\leq \frac{2\sqrt{2}L\gamma (f(x^k) - f^{\inf})}{\sqrt{L\Delta^{\inf}}} + \sqrt{2} \gamma \sqrt{ L \Delta^{\inf}}  +2 \gamma \sqrt{ L \frac{1}{M}\sum^M_{i=1} \Delta^{\inf}_i},
\end{align*}
where $\Delta^{\inf} = f^{\inf} - \frac{1}{M}\sum_{i=1}^M  f_i^{\inf}$ and $\Delta^{\inf}_i = f^{\inf} - \frac{1}{N}\sum_{j=1}^N  f_{i,j}^{\inf}$ 
\end{lemma}
\begin{proof}
Applying Lemma 6 of \cite{malinovsky2023server} for the local IG method updates in~\eqref{eqn:fixed_point_localIG}, we have 
    \begin{align*}
\frac{1}{M}\sum^M_{i=1} \frac{1}{N}\sum^{N-1}_{j=0}\left\|x^{k,j}_i - x^k \right\|^2 \leq 4L\gamma^2 \left( f(x^k) - f^{\inf} \right) + 2 \gamma^2  L \Delta^{\inf}  + 2 \gamma^2  L \frac{1}{M}\sum^M_{i=1} \Delta^{\inf}_i.
\end{align*}
Next, by Jensen's inequality,
\begin{align*}
    \frac{1}{M}\sum^M_{i=1} \frac{1}{N}\sum^{N-1}_{j=0}\left\|x^{k,j}_i - x^k \right\| 
    &\leq \sqrt{\frac{1}{M}\sum^M_{i=1} \frac{1}{N}\sum^{N-1}_{j=0}\left\|x^{k,j}_i - x^k \right\|^2 }\\
    &\leq \sqrt{4L\gamma^2 \left( f(x^k) - f^{\inf} \right) + 2 \gamma^2  L \Delta^{\inf}  + 2 \gamma^2  L \frac{1}{M}\sum^M_{i=1} \Delta^{\inf}_i}\\
    &\leq \sqrt{4L\gamma^2 \left( f(x^k) - f^{\inf} \right) + 2 \gamma^2  L \Delta^{\inf}  } + \sqrt{ 2 \gamma^2  L \frac{1}{M}\sum^M_{i=1} \Delta^{\inf}_i}.
\end{align*}
Therefore, 
\begin{align*}
    \frac{1}{M}\sum^M_{i=1} \frac{1}{N}\sum^{N-1}_{j=0}\left\|x^{k,j}_i - x^k \right\|    &\leq \frac{4L\gamma^2 \left( f(x^k) - f^{\inf} \right) + 2 \gamma^2  L \Delta^{\inf}  }{\sqrt{4L\gamma^2 \left( f(x^k) - f^{\inf} \right) + 2 \gamma^2  L \Delta^{\inf}  }}+  2 \gamma \sqrt{ L \frac{1}{M}\sum^M_{i=1} \Delta^{\inf}_i}\\
        &\leq \frac{4L\gamma^2 \left( f(x^k) - f^{\inf} \right) + 2 \gamma^2  L \Delta^{\inf}  }{\sqrt{ 2 \gamma^2  L \Delta^{\inf}  }}+  2 \gamma \sqrt{ L \frac{1}{M}\sum^M_{i=1} \Delta^{\inf}_i}\\
        &\leq \frac{2\sqrt{2}L\gamma (f(x^k) - f^{\inf})}{\sqrt{L\Delta^{\inf}}} + \sqrt{2} \gamma \sqrt{ L \Delta^{\inf}}  +2 \gamma \sqrt{ L \frac{1}{M}\sum^M_{i=1} \Delta^{\inf}_i}.
\end{align*}

\end{proof}
\begin{lemma}\label{lemma:localIG_recursion}
Let each $f_i$ be $L$-smooth, and  let $\mathcal{T}_i(x^k)=x^k- \frac{\gamma}{N} \sum_{j=0}^{N-1} \nabla f_{i,j}(x_i^{k,j})$, where the sequence $\{x_i^{k,l}\}$ is generated by 
\begin{eqnarray*}
    x_i^{k,l+1} = x_i^{k,l} - \frac{\gamma}{N}\nabla f_{i,j}(x_i^{k,l}), \quad \text{for} \quad l=0,1,\ldots,N-1,
\end{eqnarray*}
given that $x_i^{k,0}=x^k$. If $\gamma \leq \frac{1}{2L}$, and $\norm{x^{k+1}-x^k} \leq \eta$ with $\eta>0$, then 
\begin{enumerate}
    \item $x_i^{k,l} = x^k - \frac{\gamma}{N}\sum_{j=0}^{l-1} \nabla f_{i,j}(x_i^{k,l})$.  
    \item $\frac{1}{N}\sum_{j=0}^{N-1} \norm{x_i^{k+1,j}-x_i^{k,j}} \leq 2\eta $. 
    \item $\norm{\mathcal{T}_i(x^{k+1}) - \mathcal{T}_i(x^k)} \leq 2\eta$.
    \item  $\frac{1}{M}\sum^M_{i=1}\left\|   \mathcal{T}_i(x^k) - \left(x^k - \gamma\nabla f_i(x^k)\right) \right\| \leq \gamma L \frac{1}{M}\sum^M_{i=1} \frac{1}{N}\sum^{N-1}_{j=0}\left\|x^{k,j}_i - x^k \right\|$
\end{enumerate}
\end{lemma}

\begin{proof}
    The first statement derives from unrolling the recursion for $x^{k,j+1}_i$. 

    Next, we prove the second statement. 
    From the definition of $x^{k,j}_i$,
\begin{align*}
    \left\| x^{k+1,j}_{i} - x^{k,j}_{i} \right\| &= \left\| x^{k+1} - \gamma \frac{1}{N} \sum^{j-1}_{l=0} \nabla f_{i,l}(x^{k+1,l}_{i}) - \left( x^k - \gamma \frac{1}{N}\sum^{j-1}_{l=0} \nabla f_{i,l} (x^{k,l}_{i})  \right)  \right\|\\
    &\overset{\text{Triangle inequality}}{\leq} \|x^{k+1} - x^k\| + \gamma L \frac{1}{N}\sum^{j-1}_{l=0}\|x^{k+1,l}_{i} - x^{k,l}_{i}\|\\
    &\overset{\text{Triangle inequality}}{\leq} \|x^{k+1} - x^k\| + \gamma L \frac{1}{N}\sum^{N-1}_{j=0}\|x^{k+1,j}_{i} - x^{k,j}_{i}\|.
\end{align*}
Therefore, 
\begin{align*}
    \frac{1}{N}\sum^{N-1}_{j=0} \left\| x_{i}^{k+1,j} - x^{k,j}_{i} \right\| &\leq \frac{1}{N}\sum^{N-1}_{j=0} \left(\|x^{k+1} - x^k\| + \gamma \frac{1}{N}\sum^{N-1}_{j=0}\|x^{k+1,j}_{i} - x^{k,j}_{i}\| \right)\\
    & \leq \|x^{k+1} - x^k\| + \gamma L \frac{1}{N}\sum^{N-1}_{j=0}\|x^{k+1,j}_{i} - x^{k,j}_{i}\|.
\end{align*}

If $\gamma \leq \frac{1}{2L}$, then 
\begin{align*}
        \frac{1}{N}\sum^{N-1}_{j=0} \left\| x_{i,j}^{k+1} - x^k_{i,j} \right\| &\leq \frac{1}{1-\gamma L} \|x^{k+1} - x^k\|\\
        &\leq 2\|x^{k+1} - x^k\|\\
        &= 2\eta.
\end{align*}

Next, we prove the third statement. From the definition of $\mathcal{T}_i(x^{k})$ for the IG method, 
\begin{eqnarray*}
    \left\|\mathcal{T}_i(x^{k+1}) - \mathcal{T}_i(x^{k}) \right\| 
    &=& \left\| x^{k+1} - \gamma \frac{1}{N}\sum^{N-1}_{j=0} \nabla f_{i,j} (x^{k+1,j}_{i}) - \left( x^k - \gamma \frac{1}{N}\sum^{N-1}_{j=0}\nabla f_{i,j}(x^{k,j}_{i}) \right) \right\|\\
    &\overset{\text{Triangle inequality}}{\leq}& \|x^{k+1} - x^k\| + \gamma \frac{1}{N}\sum^{N-1}_{j=0}\left\| \nabla f_{i,j} (x^{k+1,j}_{i}) - \nabla f_{i,j} (x^{k,j}_{i}) \right\|\\
    &\overset{\text{$L$-smoothness of $f_{i,j}$}}{\leq}& \|x^{k+1} - x^k\| + \gamma L\frac{1}{N} \sum^{N-1}_{j=0} \|x^{k+1,j}_{i} - x^{k,j}_{i} \| \\
    & \overset{\norm{x^{k+1}-x^k} \leq \eta}{\leq} & \eta + \gamma L\frac{1}{N} \sum^{N-1}_{j=0} \|x^{k+1,j}_{i} - x^{k,j}_{i} \| \\
    & \overset{\text{The second statement}}{\leq} & \eta + \gamma L \cdot 2\eta \\
    & \overset{\gamma L \leq 1/2}{\leq}& 2\eta. 
\end{eqnarray*}

Finally, we prove the fourth statement. 
Let us consider 
\begin{align*}
    \left\| \mathcal{T}_i(x^k) - \left(x^k - \gamma\nabla f_i(x^k)\right)\right\| &= \left\| x^k - \gamma \frac{1}{N} \sum^{N-1}_{j=0} \nabla f_{i,j}(x^{k,j}_i) - \left(x^k - \gamma \nabla f_i(x^k)\right) \right\|\\
    & = \left\| \gamma \left( \frac{1}{N}\sum^{N-1}_{j=0} \nabla f_{i,j}(x^{k,j}_i) - \nabla f_i(x^k) \right)  \right\|\\
        & = \left\| \gamma \left( \frac{1}{N}\sum^{N-1}_{j=0} \nabla f_{i,j}(x^{k,j}_i) - \frac{1}{N}\sum^{N-1}_{j=0}\nabla f_{i,j}(x^k)\right)  \right\|\\
        & \overset{\text{Triangle inequality}}{\leq} \gamma \frac{1}{N}\sum^{N-1}_{j=0}\left\| \nabla f_{i,j}(x^{k,j}_i) - \nabla f_{i,j}(x^k) \right\|\\
        & \overset{\text{$L$-smoothness of $f_{i,j}$}}{\leq} \gamma L \frac{1}{N}\sum^{N-1}_{j=0}\left\|x^{k,j}_i - x^k \right\|.
\end{align*}
Therefore, 
\begin{align*}
 \frac{1}{M}\sum^M_{i=1}\left\|   \mathcal{T}_i(x^k) - \left(x^k - \gamma\nabla f_i(x^k)\right) \right\| & \leq      \frac{1}{M}\sum^M_{i=1} \gamma L \frac{1}{N}\sum^{N-1}_{j=0}\left\|x^{k,j}_i - x^k \right\|\\
     & \leq \gamma L \frac{1}{M}\sum^M_{i=1} \frac{1}{N}\sum^{N-1}_{j=0}\left\|x^{k,j}_i - x^k \right\|.
\end{align*}
\end{proof}

\subsection{Convergence Theorem for \ouralg with local IG steps}

Now, we establish the convergence theorem of  \ouralg with multiple local IG steps.

\begin{theorem}[\ouralg with local IG steps]\label{thm:FedNormEC_IG_PP_DP_multiple}
Consider \ouralg~for solving Problem~\eqref{eqn:problem} where~\Cref{assum:smooth} holds.    
Let $\mathcal{T}_i(x^k) = x^k - \gamma \frac{1}{N} \sum_{j=0}^{N-1} \nabla f_{i,j}(x^{k,j}_i),$ where  the sequence $\{x_i^{k,j}\}$ is generated by 
\(
    x^{k,j+1}_i = x^{k,j}_i - \frac{\gamma}{N} \nabla f_{i,j}(x^{k,j}_i) \quad \text{for} \quad j=0,1,\ldots,T-1,
\)
given that $x_i^{k,0}=x^k$. Furthermore, let $\beta,\alpha>0$ be chosen such that $\frac{\beta}{\alpha+R} < 1$ with $R = \max_{i \in [1,M]} \norm{v_i^0 - \frac{x^0 - \mathcal{T}_i(x^0)}{\gamma}}$.
{\color{blue}
If $\gamma = \frac{1}{2L}$ and $\eta \leq \min\left( \frac{1}{K+1}\frac{\Delta^{\inf}}{2\sqrt{2L}} , \frac{1}{6L} \frac{\beta R}{\alpha + R}\right)$, then
}
\begin{align*}
    \underset{k \in [0,K]}{\min} \Exp{\norm{\nabla f(x^k)}}  \leq& \frac{3}{K+1} \frac{f(x^0)-f^{\inf}}{\eta} + 2R + 2\sqrt{\frac{\beta^2 B}{M}(K+1)} \\ 
    & + \gamma \cdot    8L\sqrt{2L}  \sqrt{\Delta^{\inf}}  +\gamma \cdot4L\sqrt{2L} \sqrt{ \frac{1}{M}\sum^M_{i=1} \Delta^{\inf}_i}  + \eta \cdot \frac{L}{2},
\end{align*}
where 
$B= 2p (1-1/p)^2 + 2(1-p) + 2\sigma^2_{\rm DP}/p$, and $\Delta^{\inf} = f^{\inf} - \frac{1}{M}\sum_{i=1}^M f_i^{\inf} \geq 0$, and $\Delta^{\inf}_i = f^{\inf} - \frac{1}{N}\sum_{j=1}^N  f_{i,j}^{\inf}>0$
\end{theorem}

\begin{proof}
    We prove the result in the following steps.

\paragraph{Step 1) Bound  $\norm{v_i^k - \frac{x^k - \mathcal{T}_i(x^k)}{\gamma}}$ by induction, and bound $\norm{v_i^{k+1} - \frac{x^k - \mathcal{T}_i(x^k)}{\gamma}}$.}
We prove   by induction: 
$$\norm{v_i^k - \frac{x^k - \mathcal{T}_i(x^k)}{\gamma}} \leq \max_{i\in[1,M]}\norm{v_i^0 - \frac{x^0 - \mathcal{T}_i(x^0)}{\gamma}}.$$ 
We can easily show the condition when $k=0$. Next, let  $\norm{v_i^k - \frac{x^k - \mathcal{T}_i(x^k)}{\gamma}} \leq \max_{i\in[1,M]} \norm{v_i^0 - \frac{x^0 - \mathcal{T}_i(x^0)}{\gamma}}$. Then,   from~\Cref{lemma:localIG_recursion}, $\mathcal{T}_i(x^k)$ satisfies 
\begin{eqnarray*}
	\norm{\mathcal{T}_i(x^{k+1})-\mathcal{T}_i(x^k)} 
	\leq 2\eta.
\end{eqnarray*}
 Therefore, from Lemma~\ref{lemma:bound_v_and_T} with $\rho = 2$, $C = R = \max_{i\in[1,M]}\norm{v_i^0 -\frac{x^0 - \mathcal{T}_i(x^0)}{\gamma}}$, we can prove that by choosing $\frac{\beta}{\alpha + R} < 1$ and $\eta \leq \frac{\gamma\beta R}{(1+\rho) (\alpha +R)}$, $\norm{v_i^{k+1} - \frac{x^{k+1} - \mathcal{T}_i(x^{k+1})}{\gamma}} \leq R$. We complete the  proof. 

Next, from Lemma~\ref{lemma:bound_v_and_T}, $\norm{v_i^{k+1} - \frac{x^k - \mathcal{T}_i(x^k)}{\gamma}} \leq \max_{i\in[1,M]}\norm{v_i^0 - \frac{x^0 - \mathcal{T}_i(x^0)}{\gamma}}$.

\paragraph{Step 2) Bound $f(x^k)-f^{\inf}$.}
From~\Cref{lemma:descentIneq} with $G^k = \hat v^{k+1}$, 
\begin{eqnarray*}
	f(x^{k+1}) - f^{\inf} 
	& \leq & 	f(x^{k}) - f^{\inf} - \eta \norm{\nabla f(x^k)} + 2\eta \norm{\nabla f(x^k)-\hat v^{k+1}} + \frac{L\eta^2}{2} \\
	& \overset{\text{triangle inequality}}{\leq} & 	f(x^{k}) - f^{\inf} - \eta \norm{\nabla f(x^k)} + 2\eta \norm{\nabla f(x^k)- v^{k+1}}  \\
    && + 2\eta \norm{\hat v^{k+1} - v^{k+1}} + \frac{L\eta^2}{2},
\end{eqnarray*}	
where  $v^{k+1}  =  \frac{1}{M}\sum^M_{i=1} v^{k+1}_i$.
Next, since 
\begin{eqnarray*}
	\norm{\nabla f(x^k)-v^{k+1}} 
	& = & \norm{ \nabla f(x^k) -     \frac{1}{M}\sum_{i=1}^M v^{k+1}_i  } \\
	& \overset{ \text{triangle inequality} }{\leq} & \frac{1}{M}\sum_{i=1}^M  \norm{v_i^{k+1}-  \nabla f_i(x^k)} \\
    & \overset{ \text{triangle inequality} }{\leq} & \frac{1}{M}\sum_{i=1}^M  \norm{v_i^{k+1}- \frac{x^k - \mathcal{T}_i(x^k)}{\gamma}}\\
    & & + \frac{1}{M}\sum_{i=1}^M \norm{\frac{x^k - \mathcal{T}_i(x^k)}{\gamma} - \nabla f_i(x^k)},  
\end{eqnarray*}	
where $\mathcal{T}_i(x^k)=x^k - \gamma \frac{1}{N} \sum_{j=0}^{N-1} \nabla f_{i,j}(x^{k,j}_i)$, we get 
\begin{align*}
    	\norm{\nabla f(x^k)-v^{k+1}}  \leq \frac{1}{M}\sum_{i=1}^M  \norm{v_i^{k+1}- \frac{x^k - \mathcal{T}_i(x^k)}{\gamma}} + \frac{1}{\gamma}\frac{1}{M}\sum_{i=1}^M \norm{x^k - \mathcal{T}_i(x^k) - \gamma\nabla f_i(x^k)}.
\end{align*}
Plugging the upperbound for $\norm{\nabla f(x^k)-v^{k+1}}$ into the main inequality in $f(x^k)-f^{\inf}$, we obtain  
\begin{eqnarray*}
	f(x^{k+1}) - f^{\inf} 
	& \leq & 	f(x^{k}) - f^{\inf} - \eta \norm{\nabla f(x^k)} + 2\eta \frac{1}{M}\sum_{i=1}^M  \norm{v_i^{k+1} - \frac{x^k - \mathcal{T}_i(x^k)}{\gamma}}   \\
    && + \frac{2\eta}{\gamma} \frac{1}{M}\sum_{i=1}^M \norm{(x^k-\gamma\nabla f_i(x^k)) - \mathcal{T}_i(x^k)} + 2\eta \norm{\hat v^{k+1} - v^{k+1}}+ \frac{L\eta^2}{2}.
\end{eqnarray*}

By the fact that $\norm{v_i^{k+1} - \frac{x^k - \mathcal{T}_i(x^k)}{\gamma}} \leq R$ from Step 1), 
\begin{eqnarray*}
	f(x^{k+1}) - f^{\inf} 
	& \leq & 	f(x^{k}) - f^{\inf} - \eta \norm{\nabla f(x^k)} + 2\eta  R  \\
     && + \frac{2\eta}{\gamma} \frac{1}{M}\sum_{i=1}^M \norm{(x^k-\gamma\nabla f_i(x^k)) - \mathcal{T}_i(x^k)} + 2\eta \norm{\hat v^{k+1} - v^{k+1}} + \frac{L\eta^2}{2}.
\end{eqnarray*}

From~\Cref{lemma:localIG_recursion},  
\begin{eqnarray*}
	f(x^{k+1}) - f^{\inf} 
	& \leq & 	f(x^{k}) - f^{\inf} - \eta \norm{\nabla f(x^k)} + 2\eta  R  \\
     && + \frac{2\eta}{\gamma}\gamma L \frac{1}{M}\sum^M_{i=1} \frac{1}{N}\sum^{N-1}_{j=0}\left\|x^{k,j}_i - x^k \right\| + 2\eta \norm{\hat v^{k+1} - v^{k+1}} + \frac{L\eta^2}{2}.
\end{eqnarray*}

Next, from~\Cref{lemma:localIG_avg_gradient_norm}, 
\begin{eqnarray*}
	f(x^{k+1}) - f^{\inf} 
	& \leq & 	\left( 1+ \frac{4L\sqrt{2L}}{\sqrt{\Delta^{\inf}}}\gamma\eta \right)(f(x^{k}) - f^{\inf}) - \eta \norm{\nabla f(x^k)} + 2\eta  R  \\
     && + 4L \sqrt{2L} \gamma \eta  \sqrt{\Delta^{\inf}}+4L \sqrt{2L} \gamma \eta  \sqrt{ \frac{1}{M}\sum^M_{i=1} \Delta^{\inf}_i}\\
     &&+ 2\eta \norm{\hat v^{k+1} - v^{k+1}} + \frac{L\eta^2}{2}.
\end{eqnarray*}
Since
\begin{eqnarray*}
    \Exp{\norm{\hat v^{k+1} - v^{k+1}}} 
    & \leq & \frac{1}{\gamma} \Exp{\norm{ \frac{1}{M}\sum_{i=1}^M v_i^{k+1} - \hat v^{k+1} }}  \\
    & \overset{\text{\Cref{lemma:fedGD_dp_pp}}}{\leq} & \frac{1}{\gamma}\sqrt{\frac{\beta^2 B}{M}(K+1)},
\end{eqnarray*}
by taking the expectation, 
\begin{eqnarray*}
	\Exp{f(x^{k+1}) - f^{\inf}} 
	& \leq & 	\left( 1+ \frac{4L\sqrt{2L}}{\sqrt{\Delta^{\inf}}}\gamma\eta \right)\Exp{f(x^{k}) - f^{\inf}} - \eta \Exp{\norm{\nabla f(x^k)}} + 2\eta  R  \\
     && + 8L \sqrt{2L} \gamma \eta  \sqrt{\Delta^{\inf}} +4L \sqrt{2L} \gamma \eta  \sqrt{ \frac{1}{M}\sum^M_{i=1} \Delta^{\inf}_i}\\
     &&+ 2\eta \sqrt{\frac{\beta^2 B}{M}(K+1)} + \frac{L\eta^2}{2}.
\end{eqnarray*}

By applying~\Cref{lemma:localGD_convergence} with $ \eta\gamma \leq \frac{1}{K+1} \frac{\Delta^{\inf}}{4L\sqrt{2L}} $ and using the fact $(1 + \eta\gamma \frac{4L\sqrt{2L}}{\Delta^{\inf}})^{K+1} \leq \exp(\eta\gamma\frac{4L\sqrt{2L}}{\Delta^{\inf}}(K+1)) \leq \exp(1) \leq 3$ we finalize the proof.
\end{proof}

\subsection{Corollaries for \ouralg with multiple local IG steps from~\Cref{thm:FedNormEC_IG_PP_DP_multiple}}

From~\Cref{thm:FedNormEC_IG_PP_DP_multiple}, we establish the convergence bound (\Cref{corr:IG_convergence}) and utility bound (\Cref{corr:DP_PP_utility_local_IG}) for \ouralg with multiple local IG steps. 

\begin{corollary}[Convergence bound for \ouralg with multiple local IG steps]\label{corr:IG_convergence}
Consider  \ouralg~ for solving Problem~\eqref{eqn:problem} under the same setting as~\Cref{thm:FedNormEC_IG_PP_DP_multiple}. 
 Let $T>1$ (multiple local IG steps).
 If $\gamma = \frac{1}{2L (K+1)^{1/8}}$, $v_i^0\in\R^d$ is chosen such that  $\max_{i\in[1,M]} \norm{\frac{x^0 - \mathcal{T}_i(x^0)}{\gamma} - v_i^0 } = \frac{D_1}{(K+1)^{1/8}}$ with $D_1 >0$, and $\beta = \frac{D_2}{(K+1)^{5/8}}$ with $D_2 >0$, and $\eta = \frac{\hat \eta}{(K+1)^{7/8}}$ with $\hat \eta = \min\left( \frac{\Delta^{\inf}}{2\sqrt{2L}}, \frac{D_1 D_2}{4L(\alpha + D_1)}   \right)$, then  
 \begin{eqnarray*}
	 \underset{k\in [0,K]}{\min} \Exp{\norm{\nabla f(x^k)}}
	\leq  \frac{A_1}{(K+1)^{1/8}} + \frac{A_2}{(K+1)^{7/8}},
\end{eqnarray*}	  
where $A_1 = 3 \frac{f(x^0)-f^{\inf}}{\hat \eta} + 2D_1  + \frac{2\sqrt{B} D_2}{\sqrt{M}} + 8\sqrt{2L}  \sqrt{\Delta^{\inf}} + 4\sqrt{2L}  \sqrt{\frac{1}{M}\sum \limits^M_{m=1}\Delta^{\inf}_i}$ and $A_2 = \hat \eta L /2$.
\end{corollary}
\begin{proof}
The proof is analogous to the proof of Corollary~\ref{cor:conv_local_GD}.
\end{proof}

\begin{corollary}[Utility bound for \ouralg with multiple local IG steps]
\label{corr:DP_PP_utility_local_IG}
 Consider  \ouralg~ for solving Problem~\eqref{eqn:problem} under the same setting as~\Cref{thm:FedNormEC_IG_PP_DP_multiple}. 
 Let $T>1$ (multiple local IG steps),  let $\sigma_{\rm DP}= c \frac{p \sqrt{(K+1)\log(1/\delta)}}{\epsilon}$ with $c >0$ (privacy  with subsampling amplification), and let $p=\frac{\hat B}{M}$ for $\hat B\in[1,M]$ (client subsampling).
If $\beta = \frac{\hat \beta}{K+1}$ with $\hat \beta = \sqrt{\frac{3(f(x^0)-f^{\inf})}{\gamma}}\sqrt[4]{\frac{M}{B_2}}$, 
$\gamma < \frac{\Delta^{\inf} (\alpha + R)}{\sqrt{2L} \hat\beta R}$,
$\alpha = R = \cO\left( \sqrt[4]{d}\frac{\sqrt{f(x^0)-f^{\inf}}}{\sqrt{\gamma}}\sqrt[4]{\frac{B_2}{M}} \right)$ with $B_2 = 2c^2 \frac{\hat B}{M} \frac{ \log(1/\delta)}{\epsilon^2}$, and 
$\eta = \frac{1}{K+1}    \frac{\gamma}{2}\frac{\hat \beta R}{\alpha + R}$, then
\begin{eqnarray*}
     \underset{k \in [0,K]}{\min} \Exp{\norm{\nabla f(x^k)}}  = \cO\left(\Delta \sqrt[4]{\frac{d\hat B}{M^2}\frac{\log(1/\delta)}{ \epsilon^2}} + \sqrt{L}  \sqrt{\Delta^{\inf}}+\sqrt{L}  \sqrt{\frac{1}{M}\sum \limits^M_{i=1}\Delta^{\inf}_i}\right), 
\end{eqnarray*}
where $\Delta = \max(\alpha,2)\sqrt{L} \sqrt{f(x^0)-f^{\inf}}$.
\end{corollary}

\begin{proof}
The proof is analogous to the proof of Corollary~\ref{corr:DP_PP_utility_local_GD}.
\end{proof}

\end{document}